\documentclass[twoside,11pt]{article}

\usepackage{jmlr2e}
\usepackage{times}
\usepackage{natbib}
\usepackage{amsfonts}
\usepackage{amsmath}
\usepackage{latexsym}
\usepackage{color}
\usepackage{graphics}
\usepackage{enumerate}
\usepackage{amstext}
\usepackage{url}
\usepackage{epsfig}
\usepackage{multirow}
\usepackage{rotating}

\def\Rset{\mathbb{R}}

\def\Hset{\mathbb{H}}

\DeclareMathOperator*{\E}{\rm E}
\DeclareMathOperator*{\argmax}{\rm argmax}

\DeclareMathOperator{\sgn}{sgn}

\DeclareMathOperator{\cond}{cond}
\DeclareMathOperator{\Tr}{Tr}
\DeclareMathOperator{\VEC}{Vec}

\DeclareMathOperator{\epi}{epi}
\providecommand{\abs}[1]{\lvert#1\rvert}

\providecommand{\iprod}[2]{\langle#1, #2\rangle}
\providecommand{\frob}[2]{\langle#1, #2\rangle_F}
\providecommand{\frobb}[2]{\Big\langle#1, #2\Big\rangle_F}

\newcommand{\cA}{\mathcal{A}}
\newcommand{\cX}{\mathcal{X}}

\newcommand{\cM}{\mathcal{M}}
\newcommand{\mat}[1]{{\mathbf #1}}
\newcommand{\K}{\mat{K}}

\newcommand{\X}{\mat{X}}
\newcommand{\Y}{\mat{Y}}
\newcommand{\A}{\mat{A}}
\newcommand{\B}{\mat{B}}
\newcommand{\M}{\mat{M}}

\newcommand{\U}{\mat{U}}
\newcommand{\V}{\mat{V}}
\newcommand{\I}{\mat{I}}
\renewcommand{\P}{\mat{\Phi}}
\newcommand{\h}{\widehat}

\renewcommand{\a}{\mat{a}}
\renewcommand{\b}{\mat{b}}

\renewcommand{\v}{\mat{v}}
\renewcommand{\k}{\mat{k}}

\newcommand{\x}{\mat{x}}
\newcommand{\y}{\mat{y}}
\newcommand{\C}{\mat{C}}
\newcommand{\1}{\mat{1}}
\newcommand{\0}{\mat{0}}
\newcommand{\D}{\Delta}
\newcommand{\urho}{\rho_{\text{\rm  u}}}
\newcommand{\Alpha}{{\boldsymbol \alpha}}

\newcommand{\Mu}{{\boldsymbol \mu}}
\newcommand{\Nu}{{\boldsymbol \nu}}

\newcommand{\e}{\epsilon}
\newcommand{\ttop}{{\!\top\!}}
\newcommand{\tts}{\tt \small}

\newcommand{\ipsfig}[2]{\scalebox{#1}{\psfig{#2}}}

\newcommand{\set}[1]{\{#1\}}
\newcommand{\ignore}[1]{}

\jmlrheading{3}{2012}{XX-XX}{01/11}{03/12}{Corinna Cortes, Mehryar Mohri and Afshin Rostamizadeh}

\ShortHeadings{Algorithms for Learning Kernels Based on Centered Alignment}{Cortes, Mohri and Rostamizadeh}
\firstpageno{1}

\begin{document}

\title{Algorithms for Learning Kernels Based on Centered Alignment}

\author{\name Corinna Cortes \email cortes@google.com \\
       \addr Google Research \\
       76 Ninth Avenue \\
       New York, NY 10011
       \AND
       \name Mehryar Mohri \email mohri@cims.nyu.edu \\
       \addr Courant Institute and Google Research\\
       251 Mercer Street\\
       New York, NY 10012
       \AND 
       \name Afshin Rostamizadeh\thanks{A significant amount of the
       presented work was completed while AR was a graduate
       student at the Courant Institute of Mathematical Sciences and a
       postdoctoral scholar at the University of Califorinia at
       Berkeley.}
       \email rostami@google.com \\
       \addr Google Research \\
       76 Ninth Avenue \\
       New York, NY 10011 \\
       }

\editor{TBD}

\maketitle

\begin{abstract}
  This paper presents new and effective algorithms for learning
  kernels. In particular, as shown by our empirical results, these
  algorithms consistently outperform the so-called uniform combination
  solution that has proven to be difficult to improve upon in the
  past, as well as other algorithms for learning kernels based on
  convex combinations of base kernels in both classification and
  regression.  Our algorithms are based on the notion of centered
  alignment which is used as a similarity measure between kernels or
  kernel matrices. We present a number of novel algorithmic,
  theoretical, and empirical results for learning kernels based on our
  notion of centered alignment. In particular, we describe efficient
  algorithms for learning a maximum alignment kernel by showing that
  the problem can be reduced to a simple QP and discuss a one-stage
  algorithm for learning both a kernel and a hypothesis based on that
  kernel using an alignment-based regularization.  Our theoretical
  results include a novel concentration bound for centered alignment
  between kernel matrices, the proof of the existence of effective
  predictors for kernels with high alignment, both for classification
  and for regression, and the proof of stability-based generalization
  bounds for a broad family of algorithms for learning kernels based
  on centered alignment. We also report the results of
  experiments with our centered alignment-based algorithms in both
  classification and regression.
\end{abstract}


\section{Introduction}

One of the key steps in the design of learning algorithms is the
choice of the features.  This choice is typically left to the user and
represents his prior knowledge, but it is critical: a poor choice
makes learning challenging while a better choice makes it more likely
to be successful.  The general objective of this work is to define
effective methods that partially relieve the user from the requirement
of specifying the features.

For kernel-based algorithms the features are provided intrinsically
via the choice of a positive-definite symmetric kernel function
\citep{bgv,ccvv,vapnik98}. To limit the risk of a poor choice of
kernel, in the last decade or so, a number of publications have
investigated the idea of \emph{learning the kernel} from data
\citep{align-nips, chapelle_et_al, bousquet_and_herrmann, lanckriet,
jebara04, argyriou_colt, micchelli_and_pontil, lewis_et_al,
argyriou_icml, kim2006, lsk, sonnenburg, shai, zienO07,
l2reg,align,lk}. This reduces the requirement from the user to only
specifying a family of kernels rather than a specific kernel. The task
of selecting (or learning) a kernel out of that family is then
reserved to the learning algorithm which, as for standard kernel-based
methods, must also use the data to choose a hypothesis in the
reproducing kernel Hilbert space (RKHS) associated to the kernel
selected.

Different kernel families have been studied in the past, but the most
widely used one has been that of convex combinations of a finite set
of base kernels. However, while different learning kernel algorithms
have been introduced in that case, including those of
\citet{lanckriet}, to our knowledge, in the past, none has succeeded
in consistently and significantly outperforming the \emph{uniform
  combination} solution, in binary classification or regression tasks. The
uniform solution consists of simply learning a hypothesis out of the
RKHS associated to a uniform combination of the base kernels. This
disappointing performance of learning kernel algorithms has been
pointed out in different instances, including by many participants at
different NIPS workshops organized on the theme in 2008 and 2009, as
well as in a survey talk \citep{cortes09} and tutorial
\citep{tutorial11}.  The empirical results we
report further confirm this observation. Other kernel families have been
considered in the literature, including hyperkernels \citep{ong},
Gaussian kernel families \citep{micchelli_and_pontil}, or non-linear
families \citep{bach,nlk,varma}. However, the performance reported for
these other families does not seem to be consistently superior to that
of the uniform combination either.

In contrast, on the theoretical side, favorable guarantees have been
derived for learning kernels. For general kernel families, learning
bounds based on covering numbers were given by
\citet{shai}. Stronger margin-based generalization guarantees based on
an analysis of the Rademacher complexity, with only a square-root logarithmic
dependency on the number of base kernels were given by \citet{lk} for
convex combinations of kernels with an $L_1$ constraint. The
dependency of theses bounds, as well as others given for $L_q$
constraints, were shown to be optimal with respect to the number of
kernels. These $L_1$ bounds generalize those presented in
\citet{koltchinskii2008} in the context of ensembles of kernel
machines. The learning guarantees suggest that learning kernel
algorithms even with a relatively large number of base kernels could
achieve a good performance.

This paper presents new algorithms for learning kernels whose
performance is more consistent with expectations based on these
theoretical guarantees. In particular, as can be seen by our
experimental results, several of the algorithms we describe
consistently outperform the uniform combination solution. They also
surpass in performance the algorithm of \citet{lanckriet} in
classification and improve upon that of \citet{l2reg} in
regression. Thus, this can be viewed as the first series of
algorithmic solutions for learning kernels in classification and
regression with consistent performance improvements.

Our learning kernel algorithms are based on the notion of
\emph{centered alignment} which is a similarity measure between
kernels or kernel matrices. This can be used to measure the similarity
of each base kernel with the target kernel $K_Y$ derived from the
output labels. Our definition of centered alignment is close to the
uncentered kernel alignment originally introduced by
\citet{align-nips}. This closeness is only superficial however: as we
shall see both from the analysis of several cases and from
experimental results, in contrast with our notion of alignment, the
uncentered kernel alignment of \citet{align-nips} does not correlate
well with performance and thus, in general, cannot be used effectively
for learning kernels.  We note that other kernel optimization criteria
similar to centered alignment, but without the key
normalization have been used by some authors \citep{kim2006, gretton2005}.
Both the centering and the normalization are critical components of
our definition.

We present a number of novel algorithmic, theoretical, and empirical
results for learning kernels based on our notion of centered
alignment. In Section~\ref{sec:definitions}, we introduce and analyze
the properties of centered alignment between kernel functions and
kernel matrices, and discuss its benefits. In particular, the
importance of the centering is justified theoretically and validated
empirically. We then describe several algorithms based on the notion
of centered alignment in Section~\ref{sec:algorithms}.

We present two algorithms that each work in two subsequent stages
(Sections~\ref{sec:independent_align} and \ref{sec:almax}): the first
stage consists of \emph{learning} a kernel $K$ that is a non-negative
linear combination of $p$ base kernels; the second stage combines this
kernel with a standard kernel-based learning algorithm such as support
vector machines (SVMs) \citep{ccvv} for classification, or kernel
ridge regression (KRR) for
regression \citep{krr}, to select a prediction hypothesis. These two
algorithms differ in the way centered alignment is used to learn
$K$. The simplest and most straightforward to implement algorithm selects
the weight of each base kernel matrix independently, only from the
centered alignment of that matrix with the target kernel matrix. The
other more accurate algorithm instead determines these weights jointly by
measuring the centered alignment of a convex combination of base
kernel matrices with the target one.  We show that this more accurate
algorithm is very efficient by proving that the base kernel weights
can be obtained by solving a simple quadratic program (QP). We also
give a closed-form expression for the weights in the case of a linear,
but not necessarily convex, combination. Note that an alternative two-stage
technique consists of first learning a prediction hypothesis using
each base kernel and then learning the best linear combination of
these hypotheses. But, as pointed out in Section~\ref{sec:ens}, in
general, such ensemble-based techniques make use of a richer
hypothesis space than the one used by learning kernel algorithms.
In addition, we present and analyze an algorithm that uses centered alignment to
both select a convex combination kernel and a hypothesis based on that
kernel, these two tasks being performed in a single stage by solving a
single optimization problem (Section~\ref{sec:one_stage}).

We also present an extensive theoretical analysis of the notion of
centered alignment and algorithms based on that notion.  We prove a
concentration bound for the notion of centered alignment showing that
the centered alignment of two kernel matrices is sharply concentrated
around the centered alignment of the corresponding kernel functions,
the difference being bounded by a term in $O(1/\sqrt{m})$ for samples
of size $m$ (Section~\ref{sec:concentration}). Our result is simpler
and directly bounds the difference between these two relevant
quantities, unlike previous work by \citet{align-nips} (for uncentered
alignments). We also show the existence of good predictors for kernels
with high centered alignment, both for classification and for
regression (Section~\ref{sec:existence}). This result justifies the
search for good learning kernel algorithms based on the notion of
centered alignment. We note that the proofs given for similar results
in classification for uncentered alignments by
\cite{align-nips,align-unpublished} are erroneous. We also present
stability-based generalization bounds for two-stage learning kernel
algorithms based on centered alignment when the second stage is kernel
ridge regression
(Section~\ref{sec:generalization}). We further study the application
of these bounds in the case of our alignment maximization algorithm
and initiate a detailed analysis of the stability of this algorithm
(Appendix~\ref{sec:stability}).
 
Finally, in Section~\ref{sec:experiments}, we report the results of
experiments with our centered alignment-based algorithms
both in classification and regression, and compare our results with
$L_1$- and $L_2$-regularized learning kernel algorithms
\citep{lanckriet,l2reg}, as well as with the uniform kernel
combination method. The results show an improvement both over the
uniform combination and over the one-stage kernel learning
algorithms. They also demonstrate a strong correlation between the
centered alignment achieved and the performance of the
algorithm.\footnote{This is an extended version of \citep{align} with
  much additional material, including additional empirical evidence
supporting the importance of centered alignment, the description and
discussion of a single-stage algorithm for learning kernels based on
centered alignment, an analysis of unnormalized centered alignment and
the proof of the existence of good predictors for large values of
centered alignment, generalization bounds for two-stage learning
kernel algorithms based on centered alignment, and an experimental
investigation of the single-stage algorithm.}

\section{Alignment definitions}
\label{sec:definitions}

The notion of kernel alignment was first introduced by
\citet{align-nips}.  Our definition of kernel alignment is 
different and is based on the notion of centering in the feature
space. Thus, we start with the definition of centering and the
analysis of its relevant properties.

\subsection{Centered kernel functions}
\label{sec:centering}

Let $D$ be the distribution according to which training and test
points are drawn. A feature mapping $\Phi \colon \cX \to
H$ is centered by subtracting from it its expectation, that is forming it
by $\Phi - \E_{x}[\Phi]$, where $\E_{x}$ denotes the expected value of
$\Phi$ when $x$ is drawn according to the distribution $D$. Centering
a positive definite symmetric (PDS) kernel function $K\colon \cX
\times \cX \to \Rset$ consists of centering any feature mapping $\Phi$
associated to $K$.  Thus, the centered kernel $K_c$ associated to $K$
is defined for all $x, x' \in \cX$ by 
\begin{align*}
   K_c(x, x')
  & = (\Phi(x) - \E_{x}[\Phi])^\ttop (\Phi(x') - \E_{x'}[\Phi])\\
  & =  K(x, x') - \E_{x}[K(x, x')] - \E_{x'} [K(x, x')]
 +  \E_{x, x'} [K(x, x')].
\end{align*}
This also shows that the definition does not depend on the choice of
the feature mapping associated to $K$. Since $K_c(x, x')$ is defined
as an inner product, $K_c$ is also a PDS kernel.\footnote{For
convenience, we use a matrix notation for feature vectors and use
$\Phi(x)^\top\Phi(x')$ to denote the inner product between two feature
vectors and similarly $\Phi(x) \Phi(x')^\top$ for the outer product,
including in the case where the dimension of the feature space is
infinite, in which case we are using infinite matrices.} Note also
that for a centered kernel $K_c$, $\E_{x, x'}[K_c(x, x')] = 0$, that
is, centering the feature mapping implies centering the kernel
function.

\subsection{Centered kernel matrices}
\label{sec:centering_K}

Similar definitions can be given for a finite sample $S = (x_1,
\ldots, x_m)$ drawn according to $D$: a feature vector $\Phi(x_i)$
with $i \in [1, m]$ is centered by subtracting from it its empirical
expectation, that is forming it with
$\Phi(x_i) -  \overline \Phi$, where $\overline \Phi =
\frac{1}{m}\sum_{i = 1}^m \Phi(x_i)$. The kernel matrix $\K$
associated to $K$ and the sample $S$ is centered by replacing it with
$\K_c$ defined for all $i, j \in [1, m]$ by
\begin{equation}
\label{eq:Kc}
[\K_c]_{ij} = \K_{ij} - \frac{1}{m}\sum_{i = 1}^m \K_{ij}
- \frac{1}{m}\sum_{j = 1}^m \K_{ij} + \frac{1}{m^2}\sum_{i, j = 1}^m \K_{ij}.
\end{equation}
Let $\P = [\Phi(x_1), \ldots, \Phi(x_m)]^\ttop$ and $\overline \P =
[\overline \Phi, \dots, \overline \Phi]^\ttop$. Then, it is not hard
to verify that $\K_c = (\P - \overline \P)(\P - \overline \P)^\ttop$,
which shows that $\K_c$ is a positive semi-definite (PSD)
matrix. Also, as with the kernel function, $\frac{1}{m^2}\sum_{i, j =
  1}^m [\K_c]_{ij} = 0$. Let $\frob{\cdot}{\cdot}$ denote the
Frobenius product and $\| \cdot \|_F$ the Frobenius norm defined by
\begin{equation*}
\forall \A, \B \in \Rset^{m \times m}, \frob{\A}{\B} = \Tr[\A^\ttop \B] \text{ and } \| \A \|_F = \sqrt{\frob{\A}{\A}}.
\end{equation*}
Then, the following basic properties hold for centering kernel
matrices.

\begin{lemma}
\label{lemma:centering}
Let $\1 \in \Rset^{m \times 1}$ denote the vector with all entries
equal to one, and $\I$ the identity matrix.
\begin{enumerate}

\item For any kernel matrix $\K \in \Rset^{m \times m}$, the centered
  kernel matrix $\K_c$ can be expressed as follows
\begin{equation*}
\K_c = \bigg[\I - \frac{\1\1^\ttop}{m}\bigg] \K \bigg[\I - \frac{\1\1^\ttop}{m}\bigg].
\end{equation*}
\vspace{-.75cm}
\item For any two kernel matrices $\K$ and $\K'$,
\begin{equation*}
\frob{\K_c}{\K'_c}
= \frob{\K}{\K'_c}
= \frob{\K_c}{\K'}.
\end{equation*}
\end{enumerate}
\end{lemma}
\begin{proof}
  The first statement can be shown straightforwardly from the
  definition of $\K_c$ (Equation~\eqref{eq:Kc}). The second statement
  follows from
\begin{equation*}
\frob{\K_c}{\K'_c} 
= \Tr\Bigg[\bigg[\I - \frac{\1\1^\ttop}{m}\bigg] \K \bigg[\I - \frac{\1\1^\ttop}{m}\bigg] \bigg[\I - \frac{\1\1^\ttop}{m}\bigg] \K' \bigg[\I - \frac{\1\1^\ttop}{m}\bigg]\Bigg],
\end{equation*}
the fact that $[\I - \frac{1}{m} \1\1^\ttop]^2 = [\I -
\frac{1}{m} \1\1^\ttop]$, and the trace property $\Tr[\A \B] = \Tr[\B \A]$,
valid for all matrices $\A, \B \in \Rset^{m \times m}$.
\end{proof}
We shall use these properties in the proofs of the results presented in Section~\ref{sec:theory}.

\subsection{Centered kernel alignment}
\label{sec:alignment}

In the following sections, in the absence of ambiguity, to abbreviate
the notation, we often omit the variables over which an expectation is
taken.  We define the alignment of two kernel functions as follows.

\begin{definition}[Kernel function alignment]
\label{def:1}
  Let $K$ and $K'$ be two kernel functions defined over $\cX \times
  \cX$ such that $0 < \E[K_c^2] < +\infty$ and $0 <
  \E[{K'_c}^2] < +\infty$. Then, the \emph{alignment} between $K$
  and $K'$ is defined by
\begin{equation*}
\rho(K, K') 
= \frac{\E[K_c K'_c]}{\sqrt{\E[K_c^2]\E[{K'_c}^2]}} \enspace.
\end{equation*}
\end{definition}
Since $| \E[K_c K'_c] | \!\leq\! \sqrt{\E[K_c^2]\E[{K'_c}^2]}$ by the
Cauchy-Schwarz inequality, we have $\rho(K, K') \!\in\! [-1, 1]$. The
following lemma shows more precisely that $\rho(K, K') \!\in\! [0, 1]$
when $K$ and $K'$ are PDS kernels.

\begin{lemma}
\label{lemma:align}
For any two PDS kernels $K$ and $K'$, $\E[K K'] \geq 0$.
\end{lemma}
\begin{proof}
  Let $\Phi$ be a feature mapping associated to $K$ and $\Phi'$ a
  feature mapping associated to $K'$. By definition of $\Phi$ and
  $\Phi'$, and using the properties of the trace, we can write:
\begin{align*}
\E_{x, x'}[K(x, x') K'(x, x')]
& = \E_{x, x'}[\Phi(x)^\ttop\Phi(x') \Phi'(x')^\ttop\Phi'(x)]\\
& = \E_{x, x'}\big[\Tr[\Phi(x)^\ttop\Phi(x') \Phi'(x')^\ttop\Phi'(x)]\big]\\
\ignore{
& = \E_{x, x'}\big[\Tr[\Phi'(x)\Phi(x)^\ttop\Phi(x') \Phi'(x')^\ttop]\big]\\
& = \E_{x, x'} \frob{\Phi(x)\Phi'(x)^\ttop}{\Phi(x') \Phi'(x')^\ttop}\\
}
& = \frob{\E_{x}[\Phi(x)\Phi'(x)^\ttop]}{\E_{x'}[\Phi(x') \Phi'(x')^\ttop]} = \| \U \|_F^2 \geq 0,
\end{align*}
where $\U = \E_{x}[\Phi(x)\Phi'(x)^\ttop]$.
\end{proof}
The lemma applies in particular to any two centered kernels $K_c$ and
$K'_c$ which, as previously shown, are PDS kernels if $K$ and $K'$ are
PDS. Thus, for any two PDS kernels $K$ and $K'$, the following holds:
\begin{equation*}
\E[K_c K'_c] \geq 0.
\end{equation*}
We can define similarly the alignment between two kernel matrices $\K$
and $\K'$ based on a finite sample $S = (x_1, \ldots, x_m)$ drawn
according to $D$.

\begin{definition}[Kernel matrix alignment]
\label{def:2}
  Let $\K \in \Rset^{m \times m}$ and $\K' \in \Rset^{m \times
    m}$ be two kernel matrices such that $\| \K_c \|_F \neq 0$ and
  $\| \K'_c \|_F \neq 0$. Then, the \emph{alignment} between $\K$
  and $\K'$ is defined by
\begin{equation*}
\h \rho(\K, \K') 
= \frac{\frob{\K_c}{\K'_c}}{\| \K_c \|_F \| \K'_c \|_F} \enspace.
\end{equation*}
\end{definition}
Here too, by the Cauchy-Schwarz inequality, $\h \rho(\K, \K') \in
[-1, 1]$ and in fact $\h \rho(\K, \K') \geq 0$ since the Frobenius
product of any two positive semi-definite matrices $\K$ and $\K'$ is
non-negative.  Indeed, for such matrices, there exist matrices $\U$
and $\V$ such that $\K = \U\U^\ttop$ and $\K' = \V\V^\ttop$. The
statement follows from
\begin{align}
\label{eq:product}
\frob{\K}{\K'} 
= \Tr(\U\U^\ttop \V\V^\ttop) 
= \Tr\big((\U^\ttop \V)^\ttop (\U^\ttop \V)\big) 
= \| \U^\ttop \V \|_F^2
\geq 0.
\end{align}
This applies in particular to the kernel matrices of the PDS
kernels $K_c$ and $K'_c$:
\begin{align*}
\frob{\K_c}{\K'_c} \geq 0.
\end{align*}

\begin{figure}[t]
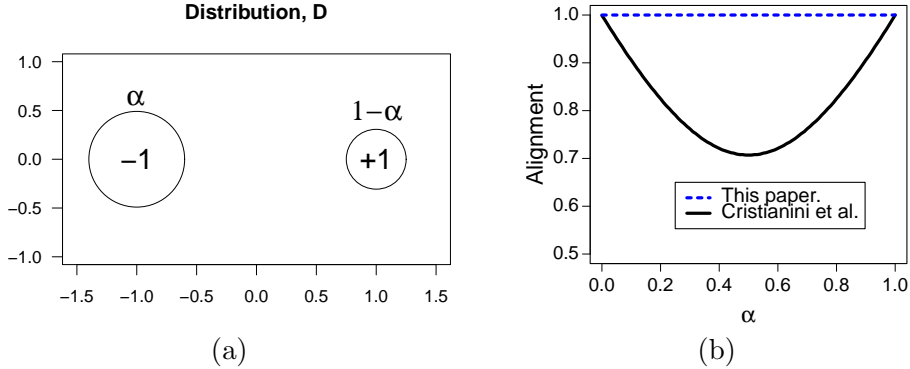

\centering
\begin{tabular}{c@{\hspace{1cm}}c}
\raisebox{.275cm}{\ipsfig{.27}{figure=d}} & \ipsfig{.27}{figure=plot}\\
(a) & (b)
\end{tabular}
\caption{(a) Representation of the distribution $D$. In this simple
  two-dimensional example, a fraction $\alpha$ of the points are at
  $(-1,0)$ and have the label $-1$. The remaining points are at
  $(1,0)$ and have the label $+1$.  (b) Alignment values computed for
  two different definitions of alignment. The solid line in black plots the
definition of alignment computed according to \citet{align-nips}  $A =
(\alpha^2 + (1 - \alpha)^2)^{1/2}$, while our definition of centered
alignment results in the straight dotted blue line $\rho = 1$.}
\label{fig:align}
\end{figure}

Our definitions of alignment between kernel functions or between
kernel matrices differ from those originally given by
\citet{align-nips,align-unpublished}:
\begin{equation*}
A = \frac{\E[K K']} {\sqrt{\E[K^2]\E[K'^2]}} \quad
\h A = \frac{\frob{\K}{\K'}}{\| \K\|_F \| \K' \|_F},
\end{equation*}
which are thus in terms of $K$ and $K'$ instead of $K_c$ and $K'_c$
and similarly for matrices. This may appear to be a technicality, but
it is in fact a critical difference. Without that centering, the
definition of alignment does not correlate well with performance.  To
see this, consider the standard case where $K'$ is the target label
kernel, that is $K'(x, x') = yy'$, with $y$ the label of $x$ and $y'$
the label of $x'$, and examine the following simple example in
dimension two ($\cX = \Rset^2$), where $K(x, x') = x \cdot x' + 1$ and
where the distribution $D$ is defined by a fraction $\alpha \in [0,
1]$ of all points being at $(-1, 0)$ and labeled with $-1$, and the
remaining points at $(1, 0)$ with label $+1$, as shown in
Figure~\ref{fig:align}.

Clearly, for any value of $\alpha \in [0, 1]$, the problem is
separable, for example with the simple vertical line going through the
origin, and one would expect the alignment to be $1$. However, the
alignment $A$ calculated using the definition of the
distribution $D$ admits a different expression. Using
\begin{align*}
  & \E[K'^2]  = 1 , \\
  & \E[K^2] = \alpha^2 \cdot 4
	  + (1 - \alpha)^2 \cdot  4
		+ 2 \alpha (1 - \alpha) \cdot 0 = 4\big(\alpha^2 + (1 -
		\alpha)^2\big) , \\
  & \E[K K'] = \alpha^2 \cdot 2
	  + (1 - \alpha)^2 \cdot  2
		+ 2 \alpha (1 - \alpha) \cdot 0 = 2\big(\alpha^2 + (1 -
		\alpha)^2\big) ,
\end{align*}
gives $A = (\alpha^2 + (1 - \alpha)^2)^{1/2}$.  Thus,
$A$ is never equal to one except for $\alpha = 0$ or $\alpha = 1$ and
for the balanced case where $\alpha = 1/2$, its value is $A =
1/\sqrt{2} \approx .707 <  1$. In contrast, with our definition,
$\rho(K, K') = 1$ for all $\alpha \in [0, 1]$ (see
Figure~\ref{fig:align}).

\begin{table}[t]
\begin{center}
\begin{sc}
\begin{tabular}{|c|c|c|c|c|c|}
\hline
    & kinematics & ionosphere & german & spambase & splice \\
    & (regr.)    & (regr.)   & (class.) & (class.) & (class.) \\
\hline
$\h \rho$ & 0.9624 & 0.9979 & 0.9439 & 0.9918 & 0.9515 \\
\hline
$\h A$ & 0.8627 & 0.9841 & 0.9390 & 0.9889 & -0.4484 \\
\hline
\end{tabular}
\end{sc}
\end{center}
\caption{The correlations of the alignment values and error-rates of
  various kernels. The top row reports the correlation of the accuracy of the
  base kernels used in Section~\ref{sec:experiments} with the centered
  alignments $\h \rho$, the bottom row the correlation
  with the non-centered alignment $\h A$.}
\label{table:centering}
\end{table}

This mismatch between $A$ (or $\h A$) and the performance values can
also be seen in real world datasets. Instances of this problem have
been noticed by \cite{meila} and \cite{pothin} who have suggested
various (input) data translation methods, and by
\cite{align-unpublished} who observed an issue for unbalanced data
sets.  Table~\ref{table:centering}, as well as
Figure~\ref{fig:centering}, give a series of empirical results in
several classification and regression tasks based on datasets taken from the UCI Machine Learning Repository
({\footnotesize \url{http://archive.ics.uci.edu/ml/}}) and Delve datasets
({\footnotesize \url{http://www.cs.toronto.edu/~delve/data/datasets.html}}).
The table and the figure illustrate the fact that the quantity $\h A$
measured with respect to several different kernels does not always
correlate well with the performance achieved by each kernel.  In fact,
for the splice classification task, the non-centered alignment is
negatively correlated with the accuracy, while a large positive
correlation is expected of a good quality measure. The centered notion
of alignment $\h \rho$ however, shows good correlation along all
datasets and is always better correlated than $\h A$.

\begin{figure}[t]
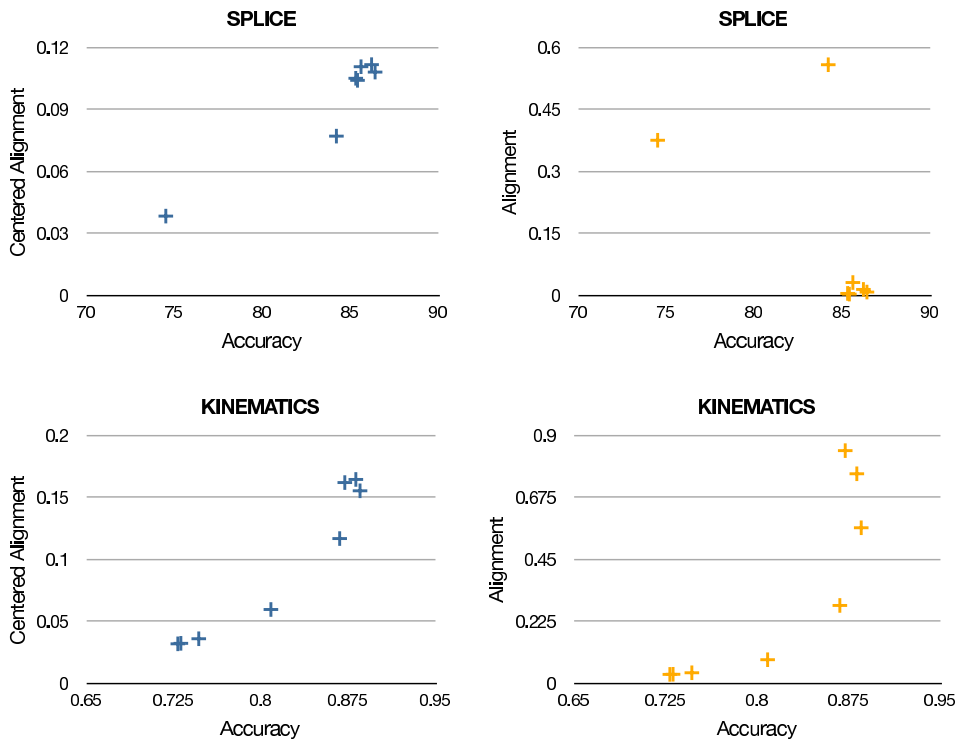

\centering
\ipsfig{.55}{figure=alignj_corr}
\caption{Detailed view of the splice and kinematics experiments
  presented in Table~\ref{table:centering}. Both the centered (plots
  in blue) and non-centered alignment (plots in orange) are plotted as
  a function of the accuracy (for the regression problem in the
  kinematics task ``accuracy'' is 1 - RMSE). It is apparent from these
  plots that the non-centered alignment can be misleading when
  evaluating the quality of a kernel. }
\label{fig:centering}
\end{figure}

The notion of alignment seeks to capture the correlation between the
random variables $K(x, x')$ and $K'(x, x')$ and one could think it
natural, as for the standard correlation coefficients, to consider the
following definition:
\begin{equation*}
\rho'(K, K') 
= \frac{\E[(K - \E[K])(K' - \E[K'])]}
{\sqrt{\E[(K - \E[K])^2]\E[(K' - \E[K'])^2]}} \enspace.
\end{equation*}
However, centering the kernel values, as opposed to centering the
feature values, is not directly relevant to linear predictions in
feature space, while our definition of alignment $\rho$ is precisely
related to that. Also, as already shown in
Section~\ref{sec:centering}, centering in the feature space implies
the centering of the kernel values, since $\E[K_c] = 0$ and
$\frac{1}{m^2}\sum_{i, j = 1}^m [\K_c]_{ij} = 0$ for any kernel $K$
and kernel matrix $\K$. Conversely, however, centering the kernel
does not imply centering in feature space. For example, consider any
kernel where all the row marginals are not all equal.

\section{Algorithms}
\label{sec:algorithms}

This section discusses several learning kernel algorithms based on the
notion of centered alignment. In all cases, the family of kernels
considered is that of non-negative combinations of $p$ base kernels
$K_k$, $k \!\in\! [1, p]$. Thus, the final hypothesis learned belongs
to the reproducing kernel Hilbert space (RKHS) $\Hset_{K_\Mu}$
associated to a kernel of the form $K_\Mu \!=\!  \sum_{k = 1}^p \mu_k
K_k$, with $\Mu \!\geq\! 0$, which guarantees that $K_\Mu$ is PDS, and
$\| \Mu \| \!=\!  \Lambda \!\geq\! 0$, for some regularization
parameter $\Lambda$.

We first describe and analyze two algorithms that both work in two
stages: in the first stage, these algorithms determine the mixture
weights $\Mu$.  In the second stage, they train a standard
kernel-based algorithm, e.g., SVMs for classification, or KRR for
regression, in combination with the kernel matrix $\K_\Mu$ associated
to $K_\Mu$, to learn a hypothesis $h \in \Hset_{K_\Mu}$. Thus, these
\emph{two-stage algorithms} differ only by their first stage, which
determines $K_\Mu$. We describe first in
Section~\ref{sec:independent_align} a simple algorithm that determines
each mixture weight $\mu_k$ independently, ({\tts align}), then, in
Section~\ref{sec:almax}, an algorithm that determines the weights
$\mu_k$s jointly ({\tts alignf}) by selecting $\Mu$ to maximize the
alignment with the target kernel. We briefly discuss in
Section~\ref{sec:ens} the relationship of such two-stage learning
algorithms with algorithms based on ensemble techniques, which also
consist of two stages. Finally, we introduce and analyze a
\emph{single-stage alignment-based algorithm} which learns $\Mu$ and
the hypothesis $h \in \Hset_{K_\Mu}$ simultaneously in
Section~\ref{sec:one_stage}.

\subsection{Independent alignment-based algorithm ({\tts align})}
\label{sec:independent_align}

This is a simple but efficient method which consists of using the
training sample to independently compute the alignment between each
kernel matrix $\K_k$ and the target kernel matrix $\K_Y = \y\y^\ttop$,
based on the labels $\y$, and to choose each mixture weight $\mu_k$
proportional to that alignment. Thus, the resulting kernel matrix is
defined by:
\begin{equation}
\label{eq:indep_align}
  \K_\Mu \propto \sum_{k=1}^p \h \rho(\K_k, \K_Y) \K_k = \frac{1}{\|
    \K_Y \|_F} \sum_{k=1}^p
  \frac{\frob{\K_k}{\K_Y}}{\| \K_k \|_F} \K_k .
\end{equation} 
When the base kernel matrices $\K_k$ have been normalized with respect
to the Frobenius norm, the independent alignment-based algorithm can
also be viewed as the solution of a joint maximization of the
unnormalized alignment defined as follows, with a $L_2$-norm constraint on
the norm of $\Mu$.

\begin{definition}[Unnormalized alignment]
\label{def:unnormalized}
  Let $K$ and $K'$ be two PDS kernels defined over $\cX \times \cX$
  and $\K$ and $\K'$ their kernel matrices for a sample of size
  $m$. Then, the \emph{unnormalized alignment} $\urho(K, K')$ between $K$ and $K'$ 
  and the \emph{unnormalized alignment} $\h \urho(\K, \K')$
  between $\K$ and $\K'$ are defined by
\begin{equation*}
  \urho(K, K') = \E_{x,x'}[K_c(x,x') K'_c(x,x')] \quad \text{and} \quad \h
  \urho(\K, \K')  = \frac{1}{m^2} \frob{\K_c}{\K'_c}.
\end{equation*}
\end{definition}
Since they are not normalized, the alignment values $a$ and $\h a$ are
no longer guaranteed to be in the interval $[0, 1]$.  However,
assuming the kernel function $K$ and labels are bounded, the
unnormalized alignment between $K$ and $K_Y$ are bounded as well.
\begin{lemma}
\label{lemma:unorm}
  Let $K$ be a PDS kernel. Assume that for all $x \in \cX$, $K_c(x,x)
  \leq R^2$ and for all output label $y$, $|y| \leq M$.  Then, the
  following bounds hold:
\begin{equation*}
  0 \leq \urho(K, K_Y) \leq M R^2 \quad \text{and} \quad 
  0 \leq \h \urho(\K, \K_Y) \leq M R^2.
\end{equation*}
\end{lemma}
\begin{proof}
  The lower bounds hold by Lemma~\ref{lemma:align} and
  Inequality~\eqref{eq:product}. The upper bounds can be obtained
  straightforwardly via the application of the Cauchy-Schwarz
  inequality:
\begin{align*}
  \urho^2(K, K_Y) 
& = \E_{(x,y), (x',y')} [ K_c(x, x') y y' ]^2  
  \leq \E_{x,x'}[ K_c^2(x, x') ] \E_{y,y'}[ y y' ]^2 
  \leq R^4 M^2\\
\h \urho(\K, \K') 
& = \frac{1}{m^2}\frob{\K_c}{\K_Y} \leq \frac{1}{m^2} \|
\K_c \|_F \| \K_y \|_F \leq \frac{mR^2mM}{m^2} = R^2M,
\end{align*}
where we used the identity $\frob{\K_c}{{\K_Y}_c}
= \frob{\K_c}{\K_Y} $ from Lemma~\ref{lemma:centering}.
\end{proof}
We will consider more generally the corresponding optimization with
an $L_q$-norm constraint on $\Mu$ with $q > 1$:
\begin{align}
\label{eq:indep_opt}
  \max_\Mu & ~~ \h \urho \Big(\sum_{k=1}^p \mu_k \K_k, \K_Y \Big) =
    \frobb{\sum_{k=1}^p \mu_k \K_k}{\K_Y} \\
  \text{subject to:} & ~~ \sum_{k = 1}^p \mu_k^q \leq \Lambda . \nonumber
\end{align}
An explicit constraint enforcing $\Mu \geq \0$ is not necessary since,
as we shall see, the optimal solution found always satisfies this constraint.
\begin{proposition}
Let $\Mu^*$ be the solution of the optimization problem
\eqref{eq:indep_opt}, then $\mu^*_k \propto \frob{\K_k}{\K_Y}^\frac{1}{q - 1}$.
\end{proposition}
\begin{proof}
The Lagrangian corresponding to the optimization (\ref{eq:indep_opt})
is defined as follows,
\begin{equation*}
  L(\Mu, \beta) = -\sum_{k = 1}^p \mu_k \frob{\K_k}{\K_Y} +
\beta(\sum_{k = 1}^p \mu_k^q - \Lambda) ,
\end{equation*}
where the dual variable $\beta$ is non-negative.
Differentiating with respect to $\mu_k$ and setting the result
to zero gives
\begin{align*}
  \frac{\partial L}{\partial \mu_k} & = -\frob{\K_k}{\K_Y} + q
\beta \mu_k^{q -1} = 0 \implies \mu_k \propto \frob{\K_k}{\K_Y}^\frac{1}{q - 1},
\end{align*}
which concludes the proof.
\end{proof}
Thus, for $q = 2$, $\mu_k \propto \frob{\K_k}{\K_Y}$ is exactly
the solution given by Equation~\eqref{eq:indep_align} modulo
normalization by the Frobenius norm of the base matrix.  Note that for
$q = 1$, the optimization becomes trivial and can be solved by simply
placing all the weight on $\mu_k$ with the largest coefficient, that
is the $\mu_k$ whose corresponding kernel matrix $\K_k$ 
has the largest alignment with the target kernel.

\subsection{Alignment maximization algorithm}
\label{sec:almax}

The independent alignment-based method ignores the correlation between
the base kernel matrices. The alignment maximization method takes
these correlations into account. It determines the mixture weights
$\mu_k$ jointly by seeking to maximize the alignment between the
convex combination kernel $\K_\Mu = \sum_{k = 1}^p \mu_k \K_k$ and the
target kernel $\K_Y = \y\y^\ttop$.

This was also suggested in the case of uncentered alignment by
\citet{align-nips,align-tech-report} and later studied by
\citet{lanckriet} who showed that the problem can be solved as a QCQP
(however, as already discussed in Section~\ref{sec:centering}, the uncentered
alignment is not well correlated with performance).
In what follows, we present even more efficient algorithms for
computing the weights $\mu_k$ by showing that the problem can be
reduced to a simple QP. We start by examining the case of a non-convex
linear combination where the components of $\Mu$ can be negative, and
show that the problem admits a closed-form solution in that case.  We
then partially use that solution to obtain the solution of the convex
combination.

\subsubsection{Linear combination}
\label{sec:align_max}

 We can assume without loss of generality that the centered base kernel matrices
${\K_k}_c$ are independent, i.e.\ no linear combination is equal to
the zero matrix,
otherwise we can select an
independent subset. This condition ensures that $\| {\K_\Mu}_c\|_F
> 0$ for arbitrary $\Mu$ and that $\h \rho (\K_\Mu, \y\y^\ttop)$ is
well defined (Definition~\ref{def:2}). By Lemma~\ref{lemma:centering},
$\frob{{\K_\Mu}_c}{{\K_Y}_c} = \frob{{\K_\Mu}_c}{\K_Y}$. Thus,
since $\| {\K_Y}_c \|_F$ does not depend on $\Mu$, the alignment
maximization problem $\max_{\Mu \in \cal M} \h \rho (\K_\Mu, \y\y^\ttop)$ can be 
equivalently written as the following optimization problem:
\begin{align}
\label{eq:2}
\max_{\Mu \in \cal M} \frac{\frob{{\K_\Mu}_c}{\y\y^\ttop}}{\| {\K_\Mu}_c \|_F},
\end{align}
where ${\cal M} = \set{\Mu\colon \| \Mu \|_2 = 1}$. A
similar set can be defined via the $L_1$-norm instead of $L_2$. As we shall
see, however, the direction of the solution $\Mu^\star$ does not
change with respect to the choice of norm. Thus, the problem can be
solved in the same way in both cases and subsequently scaled
appropriately.
Note that, by Lemma~\ref{lemma:centering}, ${\K_\Mu}_c = \U_m
\K_\Mu \U_m$ with $\U_m = \I - \1\1^\ttop/m$, thus,
\begin{equation*}
{\K_\Mu}_c 
= \U_m \Big(\sum_{k = 1}^p \mu_k \K_k\Big) \U_m
= \sum_{k = 1}^p \mu_k \U_m \K_k \U_m
= \sum_{k = 1}^p \mu_k {\K_k}_c.
\end{equation*}
Let 
\begin{equation*}
  \a = (\frob{{\K_1}_c}{\y\y^\ttop}, \ldots,
    \frob{{\K_p}_c}{\y\y^\ttop})^\ttop,
\end{equation*}
and let $\M$ denote the matrix defined by
\begin{equation*}
  \M_{kl} = \frob{{\K_k}_c}{{\K_l}_c}, 
\end{equation*}
for $k, l \in [1,
p]$. Note that, in view of the non-negativity of the
Frobenius product of symmetric PSD matrices shown in
Section~\ref{sec:alignment}, the entries of $\a$ and $\M$ are all
non-negative. Observe also that $\M$ is a symmetric PSD matrix since
for any vector $\X = (x_1, \ldots, x_p)^\ttop \in \Rset^p$,
\begin{align*}
\X^\ttop \M \X 
& = \sum_{k, l = 1}^p x_k x_l \M_{kl}\\
& = \Tr\Big[\sum_{k, l = 1}^p x_k x_l {\K_k}_c {\K_l}_c \Big]\\
& = \Tr \Big[(\sum_{k = 1}^p x_k {\K_k}_c) (\sum_{l = 1}^p x_l {\K_l}_c) \Big]
 = \| \sum_{k = 1}^p x_k {\K_k}_c \|_F^2 > 0.
\end{align*}
The strict inequality follows from the fact that the base kernels
are linearly independent.  Since this inequality holds for any
non-zero $\X$, it also shows that $\M$ is invertible.

\begin{proposition}
\label{prop:2}
The solution $\Mu^\star$ of the optimization problem (\ref{eq:2}) is
given by $\Mu^\star = \frac{\M^{-1} \a}{\| \M^{-1} \a \|}$.
\end{proposition}

\ignore{
\begin{proof}
The optimal solution $\Mu^\star$ is the solution of the following
more explicit problem:
\begin{align*}
\Mu^\star 
& = \argmax_{\| \Mu \|_2 = 1} \frac{\sum_{k = 1}^p \mu_k  \frob{{\K_k}_c}{\y\y^\ttop}}{\| \sum_{k = 1}^p \mu_k  {\K_k}_c \|_F} 
 = \argmax_{\| \Mu \|_2 = 1} \frac{(\sum_{k = 1}^p \mu_k  \frob{{\K_k}_c}{\y\y^\ttop})^2}{\| \sum_{k = 1}^p \mu_k {\K_k}_c \|_F^2}\\
& = \argmax_{\| \Mu  \|_2 = 1} \ \frac{(\Mu^\ttop \a)^2}{\Mu^\ttop \M \Mu}
= \argmax_{\| \Mu  \|_2 = 1} \ \frac{\Mu^\ttop \a\a^\ttop \Mu}{\Mu^\ttop \M \Mu}.
\end{align*}
By squaring the objective, we are implicitly assuming
${\Mu^\star}^\ttop \a \geq 0$. We will see at that at the optimum the
assumption is valid since $\M$ is PDS and ${\Mu^\star}^\ttop \a \propto
\a^\ttop \M^{-1} \a \geq 0$.  In the final equality, we recognize the
general Rayleigh quotient.  Let $\Nu = \M^{1/2} \Mu$ and $\Nu^\star =
\M^{1/2} \Mu^\star$, then, the problem can be rewritten as
\begin{align*}
\Nu^\star
& = \argmax_{\| \M^{-1/2} \Nu  \|_2 = 1} \frac{\Nu^\ttop \big[\M^{-1/2} \a\a^\ttop \M^{-1/2}\big]  \Nu}{\Nu^\ttop \Nu}.
\end{align*}
Therefore, the solution is
\begin{equation*}
\Nu^\star
 = \argmax_{\| \M^{-1/2} \Nu  \|_2 = 1} \frac{\big[\Nu^\ttop (\M^{-1/2} \a) \big]^2}{\| \Nu\|_2^2}
 = \argmax_{\| \M^{-1/2} \Nu  \|_2 = 1} \bigg[\bigg(\frac{\Nu}{\| \Nu\|}\bigg)^\ttop (\M^{-1/2} \a) \bigg]^2.
\end{equation*}
Thus, $\Nu^\star \in \VEC (\M^{-1/2} \a)$ with $\| \M^{-1/2} \Nu^\star  \|_2 = 1$.
This yields immediately
$\Mu^\star = \frac{\M^{-1} \a}{\| \M^{-1} \a \|}$.
\end{proof}
}

\begin{proof}
With the notation introduced, problem \eqref{eq:2} can be 
rewritten as
$\Mu^\star 
 = \argmax_{\| \Mu \|_2 = 1} \frac{\Mu^\ttop \a}{\sqrt{\Mu^\ttop \M \Mu}}$.
Thus, clearly, the solution must verify ${\Mu^\star}^{\ttop} \a \geq 0$.
We will square the objective and yet not enforce this condition since,
as we shall see, it will be verified by the solution we find. Therefore, 
we consider the problem
\begin{align*}
\Mu^\star 
& = \argmax_{\| \Mu  \|_2 = 1} \ \frac{(\Mu^\ttop \a)^2}{\Mu^\ttop \M \Mu}
= \argmax_{\| \Mu  \|_2 = 1} \ \frac{\Mu^\ttop \a\a^\ttop \Mu}{\Mu^\ttop \M \Mu}.
\end{align*}
In the final equality, we recognize the general Rayleigh quotient. Let
$\Nu = \M^{1/2} \Mu$ and $\Nu^\star = \M^{1/2} \Mu^\star$, then\ignore{, the
problem can be rewritten as}
\begin{align*}
\Nu^\star
& = \argmax_{\| \M^{-1/2} \Nu  \|_2 = 1} \frac{\Nu^\ttop \big[\M^{-1/2} \a\a^\ttop \M^{-1/2}\big]  \Nu}{\Nu^\ttop \Nu}.
\end{align*}
Hence, the solution is
\begin{equation*}
\Nu^\star
 = \argmax_{\| \M^{-1/2} \Nu  \|_2 = 1} \frac{\big[\Nu^\ttop \M^{-1/2} \a \big]^2}{\| \Nu\|_2^2}
 = \argmax_{\| \M^{-1/2} \Nu  \|_2 = 1} \bigg[\bigg[\frac{\Nu}{\| \Nu\|}\bigg]^\ttop \M^{-1/2} \a \bigg]^2.
\end{equation*}
Thus, $\Nu^\star \in \VEC (\M^{-1/2} \a)$ with $\| \M^{-1/2} \Nu^\star  \|_2 = 1$.
This yields immediately
$\Mu^\star = \frac{\M^{-1} \a}{\| \M^{-1} \a \|}$, which verifies
${\Mu^\star}^{\ttop} \a = \a^\ttop \M^{-1} \a/\| \M^{-1} \a \| \geq 0$
since $\M$ and $\M^{-1}$ are PSD.
\end{proof}

\subsubsection{Convex combination  ({\tts alignf})}
\label{sec:maximization-convex}

In view of the proof of Proposition~\ref{prop:2}, the alignment
maximization problem with the set ${\cal M'} = \set{\| \Mu \|_2 = 1
  \wedge \Mu \geq \0}$ can be written as
\begin{equation}
\label{eq:3}
\Mu^* = \argmax_{\Mu \in {\cal M'}} \ \frac{\Mu^\ttop \a\a^\ttop \Mu}{\Mu^\ttop \M \Mu}.
\end{equation}
The following proposition shows that the problem can be reduced
to solving a simple QP.

\begin{proposition}
\label{prop:3}
Let $\v^\star$ be the solution of the following QP:
\begin{equation}
\label{eq:10}
\min_{\v \geq \0}  \v^\ttop \M \v - 2 \v^\ttop \a.
\end{equation}
Then, the solution $\Mu^*$ of the alignment maximization problem
(\ref{eq:3}) is given by $\Mu^\star = \v^\star/\| \v^\star\|$.
\end{proposition}
\begin{proof}
Note that problem \eqref{eq:10} is equivalent to the following
one defined over $\Mu$ and $b$
\begin{equation}
\label{eq:10a}
  \min_{\substack{\Mu \geq \0, \|\Mu\|_2 = 1 \\ b > 0}}  b^2
\Mu^\ttop \M \Mu - 2 b \Mu^\ttop \a,
\end{equation}
where the relation $\v = b \Mu$ can be used to retrieve $\v$.
The optimal choice of $b$ as a function of $\Mu$ can be found by
setting the gradient of the objective function with respect to $b$ to
zero, giving the closed-form solution $b^* = \frac{\Mu^\top
\a}{\Mu^\top \M \Mu}$.  Plugging this back into \eqref{eq:10a} results
in the following optimization after straightforward simplifications:
\begin{equation*}
  \min_{\Mu \geq \0, \|\Mu\|_2 = 1} 
     - \frac{(\Mu^\top \a)^2}{\Mu^\top \M \Mu},
\end{equation*}
which is equivalent to \eqref{eq:3}. This shows that $\v^\star = b^*
\Mu^*$ where $\Mu^*$ is the solution of \eqref{eq:3} and concludes the
proof.
\end{proof}
It is not hard to see that this problem is equivalent
to solving a hard margin SVM problem, thus, any SVM solver can also be
used to solve it. A similar problem with the non-centered definition
of alignment is treated by \citet{align-tech-report-2}, but their
optimization solution differs from ours and requires cross-validation.

Also, note that solving this QP problem does not require a matrix
inversion of $\M$.  In fact, the assumption about the invertibility of
matrix $\M$ is not necessary and a maximal alignment solution can be
computed using the same optimization as that of
Proposition~\ref{prop:3} in the non-invertible case. The optimization
problem is then not strictly convex however and the alignment solution
$\Mu$ not unique.

We now further analyze the properties of the solution $\v$ of
problem \eqref{eq:10}. Let $\h \rho_0(\Mu)$ denote the partially normalized
alignment maximized by \eqref{eq:2}:
\begin{equation*}
\h \rho_0(\Mu) = \| \y\y^\ttop \|^2_F \, \h \rho(\Mu) = \frac{\frob{{\K_\Mu}_c}{\y\y^\ttop}}{\| {\K_\Mu}_c \|_F} = \frac{\Mu^\ttop \a}{\sqrt{\Mu^\ttop \M \Mu}} = \frac{\langle \Mu, \M^{-1} \a \rangle_\M}{\sqrt{\Mu^\ttop \M \Mu}} = \frac{\langle \Mu, \M^{-1} \a \rangle_\M}{\| \Mu \|_\M}.
\end{equation*}
The following proposition gives a simple expression for $\h \rho_0(\Mu)$.

\begin{proposition}
\label{prop:21}
For $\Mu = \v/\| \v \|$, with $\v \!\neq\! 0$ solution of the alignment
maximization problem~\eqref{eq:10}, the following identity
holds:
\begin{equation*}
\h \rho_0(\Mu) = \| \v \|_\M.
\end{equation*}
\end{proposition}

\begin{proof}
Since $\| \v\|_\M^2 - 2 \v^\ttop \a = \| \v\|_\M^2 - 2 \langle \v,  \M^{-1} \a \rangle_\M = \| \v - \M^{-1} \a \|_\M^2 - \| \M^{-1}\a \|_\M^2$
the optimization problem~\eqref{eq:10} can be equivalently written as
\begin{equation*}
\min_{\v \geq 0} \| \v - \M^{-1} \a \|_\M^2.
\end{equation*}
This implies that the solution $\v$ is the $\M$-orthogonal projection of $\M^{-1} \a$ over the convex set $\set{\v\colon \v \geq 0}$.
Therefore, $\v - \M^{-1} \a$ is $\M$-orthogonal to $\v$:
\begin{equation*}
\langle \v,  \v - \M^{-1} \a \rangle_\M = 0 \implies \| \v \|_\M^2 = \langle \v,  \M^{-1} \a \rangle_\M.
\end{equation*}
Thus,
\begin{equation*}
\| \v \|_\M = \frac{\langle \v,  \M^{-1} \a \rangle_\M}{\| \v \|_\M} = \frac{\langle \Mu, \M^{-1} \a \rangle_\M}{\| \Mu \|_\M} = \rho(\Mu),
\end{equation*}
which concludes the proof.
\end{proof}
Thus, the proposition gives a straightforward way of computing
$\rho_0(\Mu)$, thereby also $\rho(\Mu)$, from the $\M$-norm of the
solution vector $\v$ that $\Mu$ is derived from.

\subsection{Relationship with ensemble techniques}
\label{sec:ens}

An alternative two-stage technique for learning with multiple kernels
consists of first learning a prediction hypothesis $h_k$ using each
kernel $K_k$, $k \!\in\! [1, p]$, and then of learning the best linear
combination of these hypotheses: $h = \sum_{k = 1}^p \mu_k h_k$.  
But, such ensemble-based techniques make
use of a richer hypothesis space than the one used by learning kernel
algorithms such as that of \cite{lanckriet}. For ensemble techniques, each
hypothesis $h_k$, $k \in [1, p]$, is of the form $h_k = \sum_{i = 1}^m
\alpha_{ik} K_k(x_i, \cdot)$ for some $\Alpha_k = (\alpha_{1k},
\ldots, \alpha_{mk})^\ttop \in \Rset^m$ with 
different constraints $\|
\Alpha_k \| \leq \Lambda_k$, $\Lambda_k \geq 0$, and the final
hypothesis is of the form
\begin{equation*}
  \sum_{k = 1}^p \mu_k h_k 
   = \sum_{k = 1}^p \mu_k \sum_{j=1}^m \alpha_{ik} K_k(x_i, \cdot) 
   = \sum_{i = 1}^m \sum_{k = 1}^p \mu_k  \alpha_{ik} K_k(x_i, \cdot).
 \end{equation*}
In contrast, the general form of the hypothesis learned using
kernel learning algorithms is
\begin{equation*}
\sum_{i = 1}^m \alpha_i K_\Mu(x_i, \cdot) 
= \sum_{i = 1}^m \alpha_i \sum_{k = 1}^p \mu_k K_k(x_i, \cdot)
= \sum_{k = 1}^p \sum_{i = 1}^m \mu_k \alpha_i K_k(x_i, \cdot),
\end{equation*}
for some $\Alpha \in \Rset^m$ with $\| \Alpha \| \leq \Lambda$,
$\Lambda \geq 0$.  When the coefficients $\alpha_{ik}$ can be
decoupled, that is $\alpha_{ik} = \alpha_i \beta_k$ for some
$\beta_k$s, the two solutions seem to have the same form but they are
in fact different since in general the coefficients must obey different
constraints (different $\Lambda_k$s). Furthermore, the combination
weights $\mu_i$ are not required to be positive in the ensemble
case. We present a more detailed theoretical and empirical
comparison of the ensemble and learning kernel techniques elsewhere
\citep{ek}.

\subsection{Single-stage alignment-based algorithm}
\label{sec:one_stage}

This section analyzes an optimization based on the notion of centered
alignment, which can be viewed as the single-stage counterpart of the
two-stage algorithm discussed in Sections \ref{sec:independent_align}
- \ref{sec:almax}. 

As in Sections \ref{sec:independent_align} and \ref{sec:almax}, we
denote by $\a$ the vector $(\frob{{\K_1}_c}{\y\y^\ttop}, \ldots,
\frob{{\K_p}_c}{\y\y^\ttop})^\ttop$ and let $\M \in \Rset^{p \times
p}$ be the matrix defined by $\M_{kl} = \frob{{\K_k}_c}{{\K_l}_c}$.
The optimization is then defined by augmenting standard single-stage
learning kernel optimizations with an alignment maximization
constraint. Thus, the domain $\cM$ of the kernel combination vector
$\Mu$ is defined by:
\begin{equation*}
\cM = \set{\Mu \colon \Mu \geq \0 \wedge \| \Mu \| \leq \Lambda \wedge \rho(\K_\Mu, \y\y^\ttop) \geq \Omega},
\end{equation*}
for non-negative parameters $\Lambda$ and $\Omega$.  The alignment
constraint $\rho(\K_\Mu, \y\y^\ttop) \geq \Omega$ can be rewritten as
$\Omega \sqrt{\Mu^\ttop \M \Mu} - \Mu^\ttop \a \leq 0$, which defines
a convex region. Thus, $\cM$ is a convex subset of $\Rset^p$.

For a fixed $\Mu \!\in\! \cM$ and corresponding kernel matrix
$\K_\Mu$, let $F(\Mu, \Alpha)$ denote the objective function of the
dual optimization problem $\text{minimize}_{\Alpha \in \cA} F(\Mu,
\Alpha)$ solved by an algorithm such as SVM, KRR, or more generally
any other algorithm for which $\cA$ is a convex set and $F(\Mu,
\cdot)$ a concave function for all $\Mu \in \cM$, and $F(\cdot,
\Alpha)$ convex for all $\Alpha \in \cA$. Then, the general form of a
single-stage alignment-based learning kernel optimization is
\begin{equation*}
\min_{\Mu \in \cM} \max_{\Alpha \in \cA} F(\Mu, \Alpha).
\end{equation*}
Note that, by the convex-concave properties of $F$ and the convexity
of $\cM$ and $\cA$, von Neumann's minimax theorem applies:
\begin{equation*}
\min_{\Mu \in \cM} \max_{\Alpha \in \cA} F(\Mu, \Alpha) 
= \max_{\Alpha \in \cA} \min_{\Mu \in \cM} F(\Mu, \Alpha).
\end{equation*}
We now further examine this optimization problem in the specific case
of the kernel ridge regression algorithm. In the case of KRR, $F(\Mu,
\Alpha) = - \Alpha^{\ttop} (\K_\Mu + \lambda \I) \Alpha + 2
\Alpha^{\ttop} \y$. Thus, the max-min problem can be rewritten as
\begin{equation*}
\max_{\Alpha \in \cA} \min_{\Mu \in \cM} - \Alpha^{\ttop} (\K_\Mu + \lambda \I) \Alpha + 2 \Alpha^{\ttop} \y.
\end{equation*} Let $\b_\Alpha$ denote the vector $(\Alpha^\ttop \K_1
\Alpha, \ldots, \Alpha^\ttop \K_p \Alpha)^\top$, then the problem can be
rewritten as
\begin{equation*}
\max_{\Alpha \in \cA} - \lambda \Alpha^\ttop \Alpha + 2 \Alpha^{\ttop} \y
- \max_{\Mu \in \cM} \Mu^\ttop \b_\Alpha ,
\end{equation*}
where $\lambda = \lambda_0 m$ in the notation of equation
(\ref{eq:krr_obj}). We first focus on analyzing only the term
$-\max_{\Mu \in \cM} \Mu^\ttop \b_{\Alpha}$.
Since the last constraint in $\cM$ is convex, standard Lagrange
multiplier theory guarantees that for any $\Omega$ there exists a
$\gamma \geq 0$ such that the following optimization is equivalent to the
original maximization over $\Mu$.
\begin{align*}
\min_\Mu & \ - \Mu^\ttop \b_\Alpha + \gamma (\Omega \sqrt{\Mu^\ttop \M \Mu} - \Mu^\ttop \a)\\
\text{subject to} & \ \Mu \geq \0 \wedge \| \Mu \| \leq \Lambda \wedge
\gamma \geq 0.
\end{align*}
Note that $\gamma$ is not a variable, but rather a parameter
that will be hand-tuned. Now, again applying standard Lagrange
multiplier theory we have that for any $(\gamma \Omega) \geq 0$ there
exists an $\Omega'$ such that the following optimization is
equivalent:
\begin{align*}
\min & \ - \Mu^\ttop (\gamma \a + \b_\Alpha)\\
\text{subject to} & \ \Mu \geq \0 \wedge \| \Mu \| \leq \Lambda \wedge
\gamma \geq \0 \wedge \Mu^\ttop \M \Mu \leq \Omega'^2.
\end{align*}
Applying the Lagrange technique a final time (for any $\Lambda$ there
exists a $\gamma' \geq 0$ and for any $\Omega'^2$ there exists a
$\gamma'' \geq 0$) leads to
\begin{align*}
\min & \ - \Mu^\ttop (\gamma \a + \b_\Alpha) + \gamma' \Mu^\ttop \Mu + \gamma'' \Mu^\ttop \M \Mu \\
\text{subject to} & \ \Mu \geq \0 \wedge 
\gamma, \gamma', \gamma'' \geq 0.
\end{align*}
This is a simple QP problem. Note that the overall problem can now be
written as
\begin{equation*}
\max_{\Alpha \in \cA, \Mu \geq \0} - \lambda \Alpha^\ttop \Alpha + 2 \Alpha^{\ttop} \y +
\Mu^\ttop (\gamma \a + \b_\Alpha) - \gamma' \Mu^\ttop \Mu - \gamma'' \Mu^\ttop \M \Mu.
\end{equation*}
This last problem is not convex in $(\Alpha, \Mu)$, but the problem is
convex in each variable. \ignore{One method would consist of iteratively
solving for each variable.}
In the case of kernel ridge regression, the maximization in $\Alpha$
admits a closed form solution. Plugging in that solution yields the
following convex optimization problem in $\Mu$:
\begin{equation*}
\min_{\Mu \geq \0} 
\y^\ttop (\K_\Mu + \lambda \I)^{-1} \y 
- \gamma \Mu^\ttop \a + \Mu^\ttop (\gamma''\M  + \gamma' \I) \Mu.
\end{equation*}
Note that multiplying the objective by $\lambda$ using the
substitution $\Mu' = \frac{1}{\lambda} \Mu$ results in the following
equivalent problem,
\begin{equation*}
\min_{\Mu' \geq \0} 
\y^\ttop (\K_{\Mu'} + \I)^{-1} \y 
- \lambda^2 \gamma \Mu'^\ttop \a + \Mu'^\ttop (\lambda^3 \gamma''\M  +
  \lambda^3 \gamma' \I) \Mu',
\end{equation*}
which makes clear that the trade-off parameter $\lambda$ can be
subsumed by the $\gamma, \gamma'$ and $\gamma''$ parameters.  This
leads to the following simpler problem with a reduced number of
trade-off parameters,
\begin{equation}
\label{eq:one-opt}
\min_{\Mu \geq \0} 
\y^\ttop (\K_\Mu + \I)^{-1} \y 
- \gamma \Mu^\ttop \a + \Mu^\ttop (\gamma''\M  +
  \gamma' \I) \Mu.
\end{equation}
This is a convex optimization problem. In particular, $\Mu \mapsto
\y^\ttop (\K_\Mu + \I)^{-1} \y $ is a convex funtion by convexity of
$f\colon \M \mapsto \y^\ttop \M^{-1} \y $ over the set of positive
definite symmetric matrices. The convexity of $f$ can be seen from
that of its epigraph, which, by the property of the Schur complement,
can be written as follows \citep{boyd}:
\begin{equation*}
\epi f = \set{(\M, t)\colon \M \succ \0, \y^\ttop \M^{-1} \y \leq t} = \set{(\M, t)\colon 
\begin{pmatrix}
\M & \y\\
\y^\ttop & t
\end{pmatrix} \succeq \0, \M \succ \0}.
\end{equation*}
This defines a linear matrix inequality in $(\M, t)$ and thus a convex
set.  The convex optimization \eqref{eq:one-opt} can be solved
efficiently using a simple iterative algorithm as in \citep{l2reg}. In
practice, the algorithm converges within 10-50 iterations.  
\ignore{ 
We have
run experiments comparing this single-stage centered alignment
algorithm with the two-stage one presented in the previous
sections. Section~\ref{sec:experiments} reports the results of these
and all of our other experiments.}

\section{Theoretical results}
\label{sec:theory}

This section presents a series of theoretical guarantees related to
the notion of kernel alignment. Section~\ref{sec:concentration} proves
a concentration bound of the form $|\rho - \h \rho| \leq
O(1/\sqrt{m})$, which relates the centered alignment $\rho$ to its
empirical estimate $\h \rho$. In Section~\ref{sec:existence}, we prove
the existence of accurate predictors in both classification and
regression in the presence of a kernel $K$ with good alignment with
respect to the target kernel. Section~\ref{sec:generalization}
presents stability-based generalization bounds for the two-stage
alignment maximization algorithm whose first stage was described in
Section~\ref{sec:maximization-convex}.

\subsection{Concentration bounds for centered alignment}
\label{sec:concentration}

Our concentration bound differs from that of \citet{align-nips} both
because our definition of alignment is different and because we give a
bound directly on the quantity of interest $|\rho - \h \rho|$.
Instead, \citet{align-nips} give a bound on $|A' - \h
A|$, where $A' \neq A$ can be related to $A$ by replacing
each Frobenius product with its expectation over samples of size $m$.

The following proposition gives a bound on the essential quantities
appearing in the definition of the alignments. 

\begin{proposition}
\label{prop:1}
Let $\K$ and $\K'$ denote kernel matrices associated to the kernel
functions $K$ and $K'$ for a sample of size $m$ drawn according to
$D$. Assume that for any $x \in \cX$, $K(x, x)
\leq R^2$ and $K'(x, x) \leq R'^2$. Then, for any $\delta >
0$, with probability at least $1 - \delta$, the following
inequality holds:
\begin{equation*}
\bigg|\frac{\frob{\K_c}{\K'_c}}{m^2} - \E[K_c K'_c]\bigg| \leq
\frac{18 R^2 R'^2}{m} + 24 R^2 R'^2 \sqrt{\frac{\log \frac{2}{\delta}}{2m}}.
\end{equation*}
\end{proposition}
Note that in the case $K'(x_i,x_j) = y_i y_j$, we then have $R'^2 \leq
\max_i y_i^2$.

\begin{proof}
  The proof relies on a series of lemmas given in the Appendix. By
  the triangle inequality and in view of
  Lemma~\ref{lemma:expectation}, the following holds:
\begin{equation*}
\bigg|\frac{\frob{\K_c}{\K'_c}}{m^2} - \E[K_c K'_c]\bigg| \leq 
\bigg|\frac{\frob{\K_c}{\K'_c}}{m^2} -
\E\bigg[\frac{\frob{\K_c}{\K'_c}}{m^2}\bigg] \bigg| + \frac{18 R^2
R'^2}{m}.
\end{equation*}
Now, in view of Lemma~\ref{lemma:perturbation}, the application of
McDiarmid's inequality \citep{mcdiarmid} to
$\frac{\frob{\K_c}{\K'_c}}{m^2}$ gives for any $\e > 0$:
\begin{equation*}
\Pr \bigg[ \bigg|\frac{\frob{\K_c}{\K'_c}}{m^2} - \E\bigg[\frac{\frob{\K_c}{\K'_c}}{m^2}\bigg] \bigg| > \e \bigg] \leq 
2 \exp [-2m \e^2 / (24 R^2 R'^2)^2].
\end{equation*}
Setting $\delta$ to be equal to the right-hand side yields the 
statement of the proposition.
\end{proof}

\begin{theorem}
\label{th:main}
  Under the assumptions of Proposition~\ref{prop:1}, and further
  assuming that the conditions of the
  Definitions~\ref{def:1}-\ref{def:2} are satisfied for $\rho(K, K')$
  and $\h \rho(\K, \K')$, for any $\delta >  0$, with probability
  at least $1 - \delta$, the following inequality holds:
\begin{equation*}
|\rho(K, K') - \h \rho(\K, \K')| \leq 
18 \beta \Bigg[\frac{3}{m} + 8 \sqrt{\frac{\log \frac{6}{\delta}}{2m}}\Bigg],
\end{equation*}
with $\beta = \max(R^2 R'^2/\E[K_c^2], R^2 R'^2/\E[{K'_c}^2])$.
\end{theorem}
\begin{proof}
  To shorten the presentation, we first simplify the notation for the
  alignments as follows:
\begin{equation*}
\rho(K, K') = \frac{b}{\sqrt{a a'}} \qquad
\h \rho(\K, \K') = \frac{\h b}{\sqrt{\h a \h a'}},
\end{equation*}
with $b = \E[K_c K'_c]$, $a = \E[K_c^2]$, $a' = \E[{K'_c}^2]$ and
similarly, $\h b = (1/m^2) \frob{\K_c}{\K'_c}$, $\h a = (1/m^2)
\|\K_c\|^2$, and $\h a' = (1/m^2) \|\K'_c\|^2$. By Proposition~\ref{prop:1}
and the union bound, for any $\delta > 0$, with probability at
least $1 - \delta$, all three differences $a - \h a$, $a' - \h
a'$, and $b - \h b$ are bounded by
$\alpha = \frac{18 R^2 R'^2}{m} + 24 R^2 R'^2 \sqrt{\frac{\log \frac{6}{\delta}}{2m}}.$
Using the definitions of $\rho$ and $\h \rho$, we can write:
\begin{align*}
|\rho(K, K') - \h \rho(\K, \K')|
& = \Big| \frac{b}{\sqrt{a a'}} - \frac{\h b}{\sqrt{\h a \h a'}} \Big|
 = \Big| \frac{b \sqrt{\h a \h a'} - \h b \sqrt{a a'}}{\sqrt{a a' \h a \h a'}} \Big|\\
& = \Big| \frac{(b - \h b) \sqrt{\h a \h a'} - \h b (\sqrt{a a'} - \sqrt{\h a \h a'}) }{\sqrt{a a' \h a \h a'}} \Big|\\
& = \Big| \frac{(b - \h b)}{\sqrt{a a'}} - \h \rho(\K, \K') \frac{a a' - \h a \h a'}{\sqrt{a a'}(\sqrt{a a'} + \sqrt{\h a \h a'})} \Big|.
\end{align*}
Since $\h \rho(\K, \K') \in [0, 1]$, it follows that 
\begin{equation*}
|\rho(K, K') - \h \rho(\K, \K')|
\leq \frac{|b - \h b|}{\sqrt{a a'}} + \frac{|a a' - \h a \h a'|}{\sqrt{a a'}(\sqrt{a a'} + \sqrt{\h a \h a'})}.
\end{equation*}
Assume first that $\h a \le \h a'$. Rewriting the right-hand side to make
the differences $a - \h a$ and $a' - \h a'$ appear, we obtain:
\begin{align*}
|\rho(K, K') - \h \rho(\K, \K')|
& \leq \frac{|b - \h b|}{\sqrt{a a'}} + \frac{|(a - \h a) a' + \h a (a' - \h a')|}{\sqrt{a a'}(\sqrt{a a'} + \sqrt{\h a \h a'})}\\
& \leq \frac{\alpha}{\sqrt{a a'}}\left[1 + \frac{a' + \h a}{\sqrt{a a'} + \sqrt{\h a \h a'}}\right]
 \leq \frac{\alpha}{\sqrt{a a'}}\left[1 + \frac{a'}{\sqrt{a a'}} + \frac{\h a}{\sqrt{\h a \h a'}}\right]\\
& \leq \frac{\alpha}{\sqrt{a a'}}\left[2 + \sqrt{\frac{a'}{a}}  \right]
= \left[\frac{2}{\sqrt{a a'}} + \frac{1}{a}  \right] \alpha.
\end{align*}
We can similarly obtain $\left[\frac{2}{\sqrt{a a'}} + \frac{1}{a'}
\right] \alpha$ when $\h a' \le \h a$.  Both bounds are less than
or equal to $3 \max (\frac{\alpha}{a}, \frac{\alpha}{a'})$.
\end{proof}
Equivalently, one can set the right hand side of the high probability
statement presented in Theorem~\ref{th:main} equal to $\epsilon$ and
solve for $\delta$, which shows that $\Pr\big[|\rho(K, K') - \h
\rho(\K, \K')| > \e\big] \leq O(e^{-m \e^2})$.

\subsection{Existence of good alignment-based predictors}
\label{sec:existence}

For classification and regression tasks, the target kernel is based on
the labels and defined by $K_Y(x, x') = yy'$, where we denote by $y$
the label of point $x$ and $y'$ that of $x'$.  This section shows the
existence of predictors with high accuracy both for classification and
regression when the alignment $\rho(K, K_Y)$ between the kernel $K$
and $K_Y$ is high.

We shall assume that labels have been centered
$\E[y] = 0$ and normalized $\E[y^2] = 1$.
Denote by $h^*$ the
hypothesis defined for all $x \in \cX$ by
\begin{equation*}
h^*(x) = \frac{\E_{x'}[y'K_c(x, x')]}{\sqrt{\E[K_c^2]}}.
\end{equation*}
Observe that by definition of $h^*$, $\E_{x}[y h^*(x)] = \rho(K,
K_Y)$. For any $x \in \cX$, define $\gamma(x) =
\sqrt{\frac{\E_{x'}[K_c^2(x, x')]} {\E_{x, x'}[K_c^2(x, x')]}}$ and
$\Gamma = \max_{x} \gamma(x)$.\footnote{
If desired, one can remove the assumption of centered labels ($\E[y] = 0$) by
using the more cumbersome definitions
$h^*(x) = \frac{\E_{x'}[y'K_c(x, x')]}{\sqrt{\E[K_c^2]\E[K_{Yc}^2]}}$ and
$\gamma(x) = \sqrt{\frac{\E_{x'}[K_c^2(x, x')]} {\E_{x, x'}[K_c^2(x, x')] \E_{y,y'}[K_{Yc}^2]}}$.
} 
The following result shows that the
hypothesis $h^*$ has high accuracy when the kernel alignment is high
and $\Gamma$ not too large.\footnote{A version of this result was
  presented by \citet*{align-nips} and \citet*{align-unpublished} for
  the so-called Parzen window solution and non-centered
  kernels. However, both proofs are incorrect since they rely
  implicitly on the fact that $\max_x \big[\frac{\E_{x'}[K^2(x, x')]}
  {\E_{x, x'}[K^2(x, x')]}\big]^{\frac{1}{2}} \!=\! 1$, which can only
  hold in the trivial case where the kernel function $K^2$ is a
  constant: by definition of the maximum and expectation operators,
  $\max_x \big[\E_{x'}[K^2(x, x')]\big] \geq \E_x \big[\E_{x'}[K^2(x,
  x')]\big]$, with equality only in the constant case.}

\begin{theorem}[classification]
  Let $R(h^*) = \Pr[y h^*(x) < 0]$ denote the error of $h^*$ in
  binary classification.  For any kernel $K$ such that $0 <
  \E[K_c^2] < +\infty$, the following holds:
\begin{equation*}
R(h^*) \leq 1 - \rho(K, K_Y)/\Gamma.
\end{equation*}
\end{theorem}
\begin{proof}
Note that for all $x \in \cX$,
\begin{equation*}
|yh^*(x)| 
 = \frac{|y \E_{x'} [y' K_c(x, x')]|}{\sqrt{\E[K_c^2]}} 
 \leq \frac{\sqrt{\E_{x'}[y'^2] \E_{x'}[K^2_c(x, x')]}}{\sqrt{\E[K_c^2]}} 
 = \frac{\sqrt{\E_{x'}[K^2_c(x, x')]}}{\sqrt{\E[K_c^2]}} 
 \leq \Gamma.
\end{equation*}
In view of this inequality, and the fact that $\E_{x}[y h^*(x)] =
\rho(K, K_Y)$, we can write:
\begin{align*}
1 - R(h^*) 
& = \Pr[y h^*(x) \geq 0] \\
& = \E[\mathbf{1}_{\{y h^*(x) \geq 0\}}]\\
& \geq \E \left[\frac{y h^*(x)}{\Gamma} \mathbf{1}_{\{y h^*(x) \geq 0\}} \right]\\
& \geq \E \left[\frac{y h^*(x)}{\Gamma} \right] 
 = \rho(K, K_Y)/\Gamma,
\end{align*}
where $\mathbf{1}_\omega$ is the indicator function of the event
$\omega$.
\end{proof}
A probabilistic version of the theorem can be straightforwardly
derived by noting that by Markov's inequality, for any $\delta >
0$, with probability at least $1 - \delta$, $|\gamma(x)| \leq
1/\sqrt{\delta}$.

\begin{theorem}[regression]
  Let $R(h^*) = \E_{x}[(y - h^*(x))^2]$ denote the error of $h^*$
  in regression.  For any kernel $K$ such that $0 <  \E[K_c^2]
  < +\infty$, the following holds:
\begin{equation*}
R(h^*) \leq 2(1 - \rho(K, K_Y)).
\end{equation*}
\end{theorem}
\begin{proof}
By the Cauchy-Schwarz inequality, it follows that:
\begin{align*}
\E_{x}[{h^*}^2(x)]
& = \E_x\left[\frac{\E_{x'}[y' K_c(x, x')]^2}{\E[K_c^2]}\right]\\
& \leq \E_x\left[\frac{\E_{x'}[y'^2] \E_{x'}[K^2_c(x, x')]}{\E[K_c^2]}\right]\\
& = \frac{\E_{x'}[y'^2] \E_{x, x'}[K^2_c(x, x')]}{\E[K_c^2]}
= \E_{x'}[y'^2] = 1.
\end{align*}
Using again the fact that $\E_{x}[y h^*(x)] =
\rho(K, K_Y)$, the error of $h^*$ can be bounded as follows:
\begin{equation*}
\E[(y - h^*(x))^2]
 = \E_x[h^*(x)^2] + \E_x[y^2] - 2 \E_x[y h^*(x)]
 \leq 1 + 1 - 2 \rho(K, K_Y),
\end{equation*}
which concludes the proof.
\end{proof}
The hypothesis $h^*$ is closely related to the hypothesis $h^*_S$
derived as follows from a finite sample $S = \left((x_1, y_1), \ldots,
  (x_m, y_m)\right)$:
\begin{equation*}
  h_S(x) = \frac{\frac{1}{m}\sum_{i=1}^m y_i K_c(x, x_i)}
  {\sqrt{\frac{1}{m^2}\sum_{i,j=1}^m K_c(x_i, x_j)^2 }
  \sqrt{\frac{1}{m^2}\sum_{i,j=1}^m
  (y_i y_j)^2}}.
\end{equation*}
Note in particular that $\h \E_{x} [y h_S(x)] = \h \rho(\K, \K_\Y)$,
where we denote by $\h \E$ the expectation based on the empirical
distribution. Using this and other results of this section, it is not
hard to show that with high probability $|R(h^*) - R(h^*_S)| \leq
O(1/\sqrt{m})$ both in the classification and regression settings.

For classification, the existence of a good predictor $g^*$ based on
the unnormalized alignment $\rho_u$ (see
Definition~\ref{def:unnormalized}) can also be shown.  The corresponding
guarantees are simpler and do not depend on a term such as $\Gamma$.
However, unlike the normalized case, the loss of the predictor $g^*_S$
derived from a finite sample may not always be close to that of $g^*$.
Note that in classification, for any label $y$, $|y| = 1$, thus, by
Lemma~\ref{lemma:unorm}, the following holds: $0 \leq \urho(K, K_Y)|
\!\leq\! R^2$.  \ignore{ for any PDS kernel $K$ such that $K_c(x, x)
  \!\leq\! R^2$ for all $x$, in classification $|y| = 1$
\begin{equation*}
  |\urho(K, K_Y)| \!=\! |\E[K_c(x, x') yy']| \!\leq\! \E[|K_c(x, x')|] \!\leq\! \sqrt{\E[K_c(x, x)] \E[K_c(x', x')]}
  \!\leq\! R^2. 
\end{equation*}}
Let $g^*$ be the hypothesis defined by:
\begin{equation*}
  g^*(x) = \E_{x'}[y'K_c(x, x')].
\end{equation*}
Since $0 \leq \urho(K, K_Y)| \!\leq\! R^2$, the following theorem provides
strong guarantees for $g^*$ when the unnormalized alignment $a$ is
sufficiently large, that is close to $R^2$.

\begin{theorem}[classification]
  Let $R(g^*) = \Pr[y g^*(x) < 0]$ denote the error of $g^*$ in
  binary classification.  For any kernel $K$ such that $\sup_{x \in
  \cX} K_c(x,x) \leq R^2$, we have:
\begin{equation*}
R(g^*) \leq 1 - \urho(K, K_Y)/R^2.
\end{equation*}
\end{theorem}
\begin{proof}
Note that for all $x \in \cX$,
\begin{align*}
|yg^*(x)| = |g^*(x)| = |\E_{x'}[y'K_c(x, x')]| \leq R^2.
\end{align*}
Using this inequality, and the fact that $\E_{x}[y g^*(x)] = \urho(K,
K_Y)$, we can write:
\begin{align*}
1 - R(g^*) 
= \Pr[y g^*(x) \geq 0] 
& = \E[\mathbf{1}_{\{y g^*(x) \geq 0\}}]\\
& \geq \E \left[\frac{y g^*(x)}{R^2} \mathbf{1}_{\{y h^*(x) \geq 0\}} \right]\\
& \geq \E \left[\frac{y g^*(x)}{R^2} \right] 
= \urho(K, K_Y)/R^2,
\end{align*}
which concludes the proof.
\end{proof}
\ignore{Although the unnormalized alignment, $a$ leads to a simpler analysis
in classification, it may suffer from an unfair bias when comparing
two kernels with very different norms.  For this reason, in practice,
it would be best to first initially normalize each kernel that is
being compared.  This is exactly what is done in
Section~\ref{sec:independent_align} when deriving the independent
alignment-based weighting, which maximizes the unnormalized alignment.}

\subsection{Generalization bounds for two-stage learning kernel algorithms}
\label{sec:generalization}

This section presents stability-based generalization bounds for
two-stage learning kernel algorithms.  The proof of a stability-based
learning bound hinges on showing that the learning algorithm is
\emph{stable}, that is the pointwise loss of a learned hypothesis
does not change drastically if the training sample changes only slightly.
We refer the reader to \cite{bousquet} for a full introduction.

We present learning bounds for the case where the second stage of the
algorithm is kernel ridge regression (KRR). Similar results can be
given in classification using algorithms such as SVMs in the second
stage. Thus, in the first stage, the algorithms we examine select a
combination weight parameter $\Mu \!\in\! \cM_q \!=\!
\set{\Mu\colon\! \Mu \!\geq\! \0, \| \Mu \|_q^q \!=\! \Lambda_q}$ which
defines a kernel $K_\Mu$, and in the second stage use KRR to select a
hypothesis from the RKHS associated to $K_\Mu$.  While several of our
results hold in general, we will be more specifically interested in
the alignment maximization algorithm presented in
Section~\ref{sec:maximization-convex}.

Recall that for a fixed kernel function $K_\Mu$ with associated RKHS
$\Hset_{K_\Mu}$ and training set $S = \left((x_1,
  y_1),\ldots,(x_m,y_m)\right)$, the KRR optimization problem is
defined by the following constraint optimization problem:
\begin{equation*}
\label{eq:krr_obj}
  \min_{h \in \Hset_{K_\Mu}} G(h) = \lambda_0 \|h\|^2_{K_\Mu} 
    + \frac{1}{m} \sum_{i=1}^m (h(x_i) - y_i)^2 .
\end{equation*}
We first analyze the stability of two-stage algorithms and then use
that to derive a stability-based generalization bound
\citep{bousquet}. More precisely, we examine the pointwise difference
in hypothesis values obtained on any point $x$ when the algorithm has
been trained on two datasets $S$ and $S'$ of size $m$ that differ in
exactly one point.

In what follows, we denote by $\| \K \|_{s,t} \!=\! ( \sum_{k=1}^p \| \K_k
\|_s^t)^{1/t}$ the $(s,t)$-norm of a collection of matrices and by $\D
\Mu$ the difference $\Mu' - \Mu$ of the combination vector $\Mu'$
and $\Mu$ returned by the first stage of the algorithm by training on
$S$, respectively $S'$.

\begin{theorem}[Stability of two-stage learning kernel algorithm]
\label{th:stability}
Let $S$ and $S'$ be two samples of size $m$ that differ in exactly one
point and let $h$ and $h'$ be the associated hypotheses generated by a
two-stage KRR learning kernel algorithm with the constraint $\Mu \!\in\!
\cM_1$.  Then, for any $s, t \geq 1$ with $\frac{1}{s} + \frac{1}{r} =
1$ and any $x \in \cX$:
\begin{equation*}
  |h'(x) - h(x)| \leq \frac{2 \Lambda_1R^2
    M}{\lambda_0 m} \Big[1 + \frac{ \| \D \Mu \|_s \| \K_c \|_{2,t} }{2
    \lambda_0} \Big], 
\end{equation*}
where $M$ is an upper bound on the target labels and $R^2 =
\sup_{\substack{k \in [1, p]\\ \!\!\!\!\! x  \in \cX}} K_k(x, x) $.
\end{theorem}
\begin{proof}
The KRR algorithm returns the hypothesis $h(x) = \sum_{i=1}^m \alpha_i
K_\Mu(x_i, x)$, where $\Alpha = (\K_\Mu + m \lambda_0 \I)^{-1} \y$.
Thus, this hypothesis is parametrized by the kernel
weight vector $\Mu$, which defines the kernel function, and
the sample $S$, which is used to populate the kernel matrix, and will
be explicitly denoted $h_{\Mu,S}$.  To estimate the stability of the
overall two-stage algorithm, $\D h_{\Mu, S} = h_{\Mu', S'} - h_{\Mu,
S}$, we use the decomposition 
\begin{equation*}
\D h_{\Mu, S} 
 = (h_{\Mu', S'} - h_{\Mu', S}) + (h_{\Mu', S} - h_{\Mu, S})
\end{equation*}
and bound each parenthesized term separately.  
The first parenthesized term measures the pointwise stability of KRR
due to a change of a single training point with a fixed kernel. This
can be bounded using Theorem 2 of \citep{l2reg}. Since, for all $x \in
\cX$, $K_\Mu(x, x) = \sum_{k = 1}^p \mu_k K_k(x, x) \leq R^2 \sum_{k =
  1}^p \mu_k \leq \Lambda_1 R^2$, using that theorem yields the
following bound:
\begin{equation*}
  \forall x \in \cX, \quad 
  | h_{\Mu, S'}(x) - h_{\Mu, S}(x) |
  \leq \frac{2 \Lambda_1R^2 M}{\lambda_0 m}.
\end{equation*}
The second parenthesized term measures the pointwise difference of the
hypotheses due to the change of kernel from ${\K}_{\Mu'}$ to
${\K}_{\Mu}$ for a fixed training sample when using KRR. By
Proposition 1 of \citep{approx}, this term can be bounded as follows:
\begin{equation*}
\forall x \in \cX, \abs{h_{\Mu',S}(x) - h_{\Mu, S}(x)} 
\leq \frac{\Lambda_1R^2 M}{\lambda_0^2 m} \|\K_{\Mu'} - \K_{\Mu} \| .
\end{equation*}
The term $\| \K_{\Mu'} - \K_\Mu \|$ can be bounded using H\"older's
inequality as follows:
\begin{align*}
\| \K_{\Mu'} - \K_\Mu \|  
& = \| \sum_{k = 1}^p (\D \mu_k) \K_k \|
 \leq \sum_{k = 1}^p |\D \mu_k|\ \| \K_k \| 
  \leq \| \D \Mu \|_s \| \K \|_{2,t},
\end{align*}
which completes the proof.
\end{proof}
The pointwise stability result just presented can be used directly to derive a generalization
bound for two-stage learning kernel algorithms as in \citep{bousquet}.

For a hypothesis $h$, we denote by $R(h)$ its generalization error
and by $\h R(h)$ its empirical error on a $S = \left((x_1,
  y_1),\ldots,(x_m,y_m)\right)$:
\begin{equation*}
R(h) = \E_{x,y}[(h_S(x) - y)^2]  \quad \h R(h) = \frac{1}{m} \sum_{i=1}^m (h_S(x_i) - y_i)^2.
\end{equation*}

\begin{theorem}[Stability-based generalization bound]
\label{th:lbound}
  Let $h_S$ denote the hypothesis returned by a two-stage KRR kernel
  learning algorithm with the constraint $\Mu \!\in\!  \cM_1$ when
  trained on sample $S$. For any $s, t \!\geq\! 1$ with $\frac{1}{s} +
  \frac{1}{r} = 1$, with probability at least $1 - \delta$ over
  samples $S$ of size $m$, the following bound holds:
\begin{equation*}
  R(h_S) \leq \h R(h_S) + \frac{2 M_1 M_2}{m} + \Big(
1 +  \frac{16 M_2}{\ M_1} \Big) \frac{M_1 M_2}{4} \sqrt{\frac{ \log \frac{1}{\delta}}{2m} },
\end{equation*}
with $M_1 = 2 \Big[ 1 + \sqrt{\frac{\Lambda_1 R^2}{\lambda_0}} \Big] M$ and
$M_2 = \frac{2 \Lambda_1 R^2 }{\lambda_0}
\Big[1 + \frac{ \| \D \Mu \|_s \| \K_c \|_{2,t} }{2
    \lambda_0} \Big] M$.
\end{theorem}
\begin{proof}
Since $h_S$ is the minimizer of the objective \eqref{eq:krr_obj} and
since $\0$ belongs to the hypothesis space, 
\begin{equation*}
G(h_S) \leq G(\0) = \frac{1}{m} \sum_{i=1}^m (0 - y_i)^2 \leq M^2.
\end{equation*}
Furthermore, since the mean squared loss is non-negative, we can
write: $\lambda_0 \| h_S \|^2_{K_\Mu} \leq G(h_S)$. Therefore,
$\|h_S\|^2_{K_\Mu} \leq \frac{M^2}{\lambda_0}$. By the reproducing
property, for any $x \in \cX$,
\begin{align*}
|h_S(x)| 
= |\iprod{h_S}{K_\Mu(x, \cdot)}_{K_\Mu}|
& \leq \|h_S\|_{K_\Mu} \sqrt{K_\Mu(x,x)} \\
& = \sqrt{\frac{M}{\lambda_0}} \sqrt{\sum_{k=1}^p \mu_k K_k(x,x)} \\
& \leq \sqrt{\frac{M}{\lambda_0}} \sqrt{\| \Mu \|_1 R^2 }
\leq R M \sqrt{\frac{\Lambda_1}{\lambda_0}}.
\end{align*}
Thus, for all $x \in \cX$ and $y \in [-M, M]$, the squared loss can be
bounded as follows
\begin{equation*}
|h_S(x) - y| \leq \Big( M + R M \sqrt{\frac{\Lambda_1}{\lambda_0}} \Big) = \frac{M_1}{2}.
\end{equation*}
This implies that the squared loss is $M_1$-Lipschitz and by
Theorem~\ref{th:stability} that the algorithm is stable with a uniform
stability parameter $\beta \leq \frac{M_1 M_2}{m}$ bounded as follows:
\begin{equation*}
  |(h_{S'}(x) - y)^2 - (h_{S}(x) - y)^2| \leq M_1
  |h_{S'}(x)  - h_{S}(x) | \leq  \frac{M_1 M_2}{m}.
\end{equation*}
The application of Theorem 12 of \citep{bousquet} with the bound on
the loss $\frac{M_1}{2}$ and the uniform stability parameter $\beta$
directly yields the statement.
\end{proof}
The inequality just presented holds for all two-stage learning kernel
algorithms. To determine its convergence rate, the term $\| \D \Mu
\|_s \| \K_c \|_{2,t}$ must be bounded. Let $s = 1$ and $t = \infty$,
and assume that the base kernels $\K_k$, $k \in [1, p]$, are
trace-normalized as in our experiments (Section~\ref{sec:algorithms}),
then a straightforward bound can be given for this term:
\begin{equation*}
  \| \D \Mu \|_1  \| \K_c \|_{2,\infty} \leq (\| \Mu' \|_1 + \| \Mu
  \|_1) \max_{k \in [1, k]} \| {\K_k}_c \|_2 \leq
  \max_{k \in [1, k]} 2 \Lambda_1 \Tr[{\K_k}_c] \leq 2 \Lambda_1.
\end{equation*}
Thus, in the statement of Theorem~\ref{th:lbound}, $M_2$ can be
replaced with $\frac{2 \Lambda_1 R^2 }{\lambda_0} \Big[1 +
\frac{\Lambda_1}{\lambda_0} \Big] M$ and, for $\Lambda_1$ and
$\lambda_0$ constant, the learning bound converges in $O(1/\sqrt{m})$.

The straightforward upper bound on $\| \D \Mu \|_s \| \K_c \|_{2,t}$
applies to all such two-stage learning kernel algorithms. For a
specific algorithm, finer or more favorable bounds could be derived.
We have initiated this study in the specific case of the alignment
maximization algorithm. The result given in Proposition~\ref{prop:20}
(Appendix~\ref{sec:stability}) can be used to bound $\| \D \Mu \|_2$
and thus $\| \D \Mu \|_2 \| \K_c \|_{2,2}$.

Note that in the specific case of the alignment maximization algorithm, if
$\Mu^*$ is the solution obtained for the constraint $\Mu \!\in\!
\cM_2$, then it is also the alignment maximizing solution found in the
set $\Mu \!\in\! \cM_1$ with $\Lambda_1 \!=\! \| \Mu^* \|_1 \!\leq\!
\sqrt{p} \| \Mu \|_2 \!\leq\! \sqrt{p} \Lambda_2$. This makes the
dependence on $p$ explicit in the case of the constraint $\Mu \!\in\!
\cM_2$.

\section{Experiments}
\label{sec:experiments}

This section compares the performance of several learning kernel
algorithms for classification and regression. We compare the
alignment-based two-stage learning kernel algorithms {\tts
  align} and {\tts alignf}, as well as the single-stage algorithm
presented in Section~\ref{sec:algorithms} with the following
algorithms:\\

\noindent {\bf Uniform combination ({\tts unif})}: this is the most
straightforward method, which consists of choosing equal mixture
weights, thus the kernel matrix used is, 
\begin{equation*}
  \K_\Mu = \frac{\Lambda}{p} \sum_{k=1}^p \K_k .
\end{equation*}
Nevertheless, improving upon the performance of
this method has been surprisingly difficult for standard (one-stage)
learning kernel algorithms \citep{cortes09, tutorial11}.\\

\noindent {\bf Norm-1 regularized combination} ({\tts l1-svm}): this
algorithm optimizes the SVM objective
  \begin{align*}
     \min_{\Mu} \max_\Alpha & 
      ~~ 2 \Alpha^\ttop \1 - \Alpha^\ttop \Y^\ttop \K_\Mu \Y \Alpha  \nonumber \\
      \text{subject to:} & ~ \Mu \geq \0, \Tr[\K_\Mu] \leq \Lambda,
                           \Alpha^\ttop \y = 0,
                          \0 \leq \Alpha \leq \C, \nonumber
  \end{align*} 
  as described by \citet{lanckriet}. Here, $\Y$ is the diagonal matrix
  constructed from the labels $\y$ and $\C$ is the regularization
  parameter of the SVM.\\

\noindent   {\bf Norm-2 regularized combination} ({\tts l2-krr}): this algorithm
optimizes the kernel ridge regression objective 
\begin{align*}
\min_{\Mu} \max_\Alpha & -\lambda \Alpha^\ttop \Alpha - \Alpha^\ttop
\K_\Mu \Alpha + 2 \Alpha^\ttop \y  \\ 
\text{subject to:} & ~ \Mu \geq \0,  \|\Mu - \Mu_0\|_2 \leq \Lambda .
\end{align*} 
The $L_2$ regularized method is used for regression since it is shown
in \citep{l2reg} to outperform the alternative $L_1$ regularized
method in similar settings. 
Here, $\lambda$ is the regularization
parameter of KRR and $\Mu_0$ is an additional regularization
parameter for the kernel selection.

  In all experiments, the error measures reported are for 5-fold cross
validation, where, in each trial, three folds are used for training,
one used for validation, and one for testing. For the two-stage
methods, the same training and validation data is used for both stages
of the learning.  The regularization parameter $\Lambda$ is chosen via
a grid search based on the performance on the validation set, while
the regularization parameters $\C$ for {\tts l1-svm} and $\lambda$ for
{\tts l2-krr} are fixed since only
the ratios $\C / \Lambda$ and $\lambda / \Lambda$ are important.  More
explicitly, for the KRR algorithm, scaling the vector $\Mu$ by
$\Lambda$ results in a scaled dual solution: $\Alpha = (\K_\Mu \Lambda
+ \lambda \I)^{-1} \y = \Lambda^{-1}(\K_\Mu + \frac{\lambda}{\Lambda}
\I)^{-1} \y$.  In turn, we see that the primal solution $h(x) =
\sum_{i=1}^m \Lambda^{-1} \alpha_i \Lambda K_\Mu(x,x_i) =
\sum_{i=1}^m \alpha_i K_\Mu(x,x_i)$ is equivalent to the solution of
the KRR algorithm that uses a regularization parameter equal to
$\lambda / \Lambda$ without scaling $\Mu$ and, thus, it suffices to vary only
one regularization parameter. In the case of SVMs, the scale of the hypothesis
does not change its sign (or the binary prediction) and thus the same property
can be shown to hold.  The $\Mu_0$ parameter is set to zero in our
experiments. 
\ignore{
in
Section~\ref{sec:general_kernels}, and is chosen to be uniform in
Section~\ref{sec:rank1_kernels}.
}

\begin{table}[t]
\centering
\begin{sc}
\begin{center}
\begin{tabular}{|  c  |  c  |  c  |  c  |  c  |  c  |}
\hline
 & kinematics & ionosphere & german & spambase & splice \\
\hline
size & 1000 & 351 & 1000 & 1000 & 1000 \\
\hline
$\gamma$ & -3, 3 & -3, 3 & -4, 3 & -12, -7  & -9, -3 \\
\hline
\multirow{2}{*}{\tt unif} 
  & $.138 \pm.005$ & $.479  \pm.033$ & $.259  \pm.018$ & $.187 \pm.028$ & $.152  \pm .022$ \\ 
  & $.158  \pm.013$ & $.246  \pm.033$ & $.089  \pm.008$ & $.138 \pm.031$ & $.122  \pm.011$ \\
\hline
\multirow{2}{*}{\tt 1-stage} 
  & $.137  \pm.005$ & $.470  \pm.032$ & $.260  \pm .026$ & $.209 \pm .028$ & $.153  \pm .025$ \\
  & $.155  \pm.012$ & $.251  \pm.035$ & $.082  \pm.003$ & $.099 \pm.024$ & $.105  \pm.006$ \\
\hline
\multirow{2}{*}{\tt align} 
	& $.125  \pm.004$ & $.456  \pm.036$ & $.255  \pm .015$ & $.186 \pm .026$ & $.151  \pm .024$ \\
  & $.173  \pm.016$ & $.261  \pm .040$ & $.089  \pm.008$ & $.140 \pm.031$ & $.123  \pm.011$ \\
\hline
\multirow{2}{*}{\tt alignf} 
  & $.115  \pm.004$ & $.444  \pm.034$ & $.242 \pm .015$ & $.180 \pm .024$ & $.139  \pm .013$ \\
  & $.176  \pm.017$ & $.278  \pm.057$ & $.093  \pm.009$ & $.146 \pm.028$ & $.124  \pm.011$ \\
\hline
\end{tabular}
\end{center}
\hspace{1.3cm} Regression \hspace{2.5cm} Classification
\end{sc}
\caption{Error measures (top) and alignment values (bottom) for 
{\tts unif}, {\tts 1-stage} ({\tts l2-krr} or {\tts l1-svm}), {\tts
 align} and {\tts alignf} with kernels built from linear
combinations of Gaussian base kernels.  The choice of $\gamma_0,
\gamma_1$ is listed in the row labeled $\gamma$ and the total size of
the dataset used is listed under {\sc size}. The results
are shown with $\pm1$ standard deviation measured by 5-fold
cross-validation. Further measures of significance are shown in
Appendix~\ref{sec:sig}, Table~\ref{table:experiments_sig}.}
\label{table:experiments}
\vspace{-1cm}
\end{table}

\subsection{General kernel combinations}
\label{sec:general_kernels}

In the first set of experiments, we consider combinations of Gaussian
kernels of the form 
\begin{equation*}
  \K_\gamma(\x_i, \x_j) = \exp(-\gamma \|\x_i - \x_j\|^2),
\end{equation*}
with varying bandwidth parameter $\gamma \in
\{2^{\gamma_0}, 2^{\gamma_0+1},\ldots,2^{1-\gamma_1},
2^{\gamma_1}\}$. The values $\gamma_0$ and $\gamma_1$ are chosen such
that the base kernels are sufficiently different in alignment and
performance.  Each base kernel is centered and
normalized to have trace one.\ignore{, and in the regression setting
the labels are also centered.} We test the algorithms on several
datasets taken from the UCI Machine Learning Repository
({\footnotesize \url{http://archive.ics.uci.edu/ml/}}) and Delve 
({\footnotesize \url{http://www.cs.toronto.edu/~delve/data/datasets.html}}).

Table~\ref{table:experiments} summarizes our results.
For regression, we
compare against the {\tts l2-krr} method and report RMSE.
For classification, we compare against the {\tts l1-svm} method and
report the misclassification percentage.   In general,
we see that performance and alignment are well correlated.  In all
datasets, we see improvement over the uniform combination as
well as the one-stage kernel learning algorithms.  Note that although
the {\tts align} method often increases the alignment of the final
kernel, as compared to the uniform combination, the {\tts alignf}
method gives the best alignment since it directly maximizes
this quantity.  Nonetheless, {\tts align} provides an inexpensive
heuristic that increases the alignment and performance of the final
combination kernel.

\begin{figure}
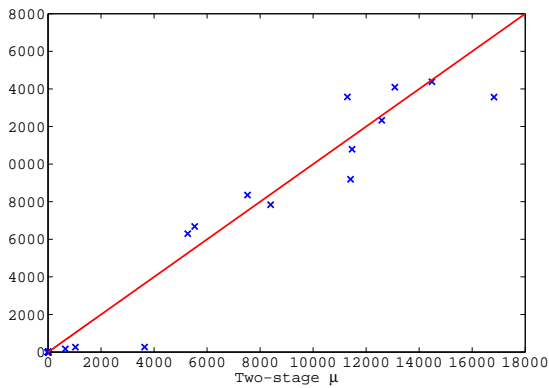

\begin{center}
\ipsfig{.33}{figure=comparison}
\end{center}
\caption{A scatter plot comparison of the different kernel combination
weight values obtained by optimally tuned one-stage and two-stage
algorithms on the kinematics dataset.}
\label{fig:comparison}
\end{figure}

In our experiments with the one-stage KRR algorithm presented in
Section \ref{sec:one_stage}, there was no significant improvement
found over the two-stage {\tts alignf} algorithm with respect to the
kinematics and ionosphere datasets.  In fact, for
optimally cross-validated parameters $\gamma, \gamma'$ and $\gamma''$
the solution combination weights were found to closely coincide with
the {\tts alignf} solution (see Figure \ref{fig:comparison}).  This
would suggest the use of the two-stage algorithm over the one-stage,
since there are fewer parameters to tune and the problem can be solved
as a standard QP.

To the best of our knowledge, these are the first kernel
combination experiments for alignment with general base
kernels.  Previous experiments seem to have dealt exclusively with
rank-one base kernels built from the eigenvectors of a single kernel
matrix \citep{align-nips}. In the next section, we also examine rank-one
kernels, although not generated from a spectral decomposition.

\subsection{Rank-one kernel combinations}
\label{sec:rank1_kernels}

In this set of experiments we use the sentiment analysis dataset
version~1 from \citet{Blitzer07Biographies}: \emph{books}, \emph{dvd},
\emph{electronics} and \emph{kitchen}. Each domain has 2,000 examples.
In the regression setting, the goal is to predict a rating between 1
and 5, while for classification the goal is to discriminate positive
(ratings $\geq 4$) from negative reviews (ratings $\leq 2$). We use
rank-one kernels based on the 4,000 most frequent bigrams. The $k$th
base kernel, $\K_k$, corresponds to the $k$th bigram count $\v_k$,
$\K_k = \v_k \v_k^\ttop$. Each base kernel is normalized to have trace
one and the labels are centered.

\begin{table}[t]
\centering
\begin{sc}
\begin{tabular}{|  c  |  c  |  c  |  c  |  c  |}
\hline
 & books & dvd & elec & kitchen \\
\hline
\multirow{2}{*}{\tts unif} 
  & $1.442 \pm .015$ & $1.438 \pm .033$ & $1.342 \pm .030$ & $1.356 \pm .016$ \\
  & $0.029 \pm .005$ & $0.029 \pm .005$ & $0.038 \pm .002$ & $0.039 \pm .006$ \\
\hline
\multirow{2}{*}{\tts l2-krr} 
  & $1.410 \pm .024$ & $1.423 \pm .034$ & $1.318 \pm .033$ & $1.333 \pm .015$ \\
  & $0.036 \pm .008$ & $0.036 \pm .009$ & $0.050 \pm .004$ & $0.056 \pm .005$ \\
\hline
\multirow{2}{*}{\tts align} 
  & $1.401 \pm .035$ & $1.414 \pm .017$ & $1.308 \pm .033$ & $1.312 \pm .012$ \\
  & $0.046 \pm .006$ & $0.047 \pm .005$ & $0.065 \pm .004$ & $0.076 \pm .008$ \\
\hline
\end{tabular} \\[.25cm]
Regression \\[.25cm]
\vspace{0.5cm}
\begin{tabular}{|  c  |  c  |  c  |  c  |  c  |}
\hline
 & books & dvd & elec & kitchen \\
\hline
\multirow{2}{*}{\tts unif} 
  & $0.258 \pm .017$ & $0.243 \pm .015$ & $0.188 \pm .014$ & $0.201 \pm .020$ \\
  & $0.030 \pm .004$ & $0.030 \pm .005$ & $0.040 \pm .002$ & $0.039 \pm .007$ \\
\hline
\multirow{2}{*}{\tts l1-svm} 
  & $0.286 \pm .016$ & $0.292 \pm .025$ & $0.238 \pm .019$ & $0.236 \pm .024$ \\
  & $0.030 \pm .011$ & $0.033 \pm .014$ & $0.051 \pm .004$ & $0.058 \pm .007$ \\
\hline
\multirow{2}{*}{\tts align} 
  & $0.243 \pm .020$ & $0.214 \pm .020$ & $0.166 \pm .016$ & $0.172 \pm .022$ \\
  & $0.043 \pm .003$ & $0.045 \pm .005$ & $0.063 \pm .004$ & $0.070 \pm .010$ \\
\hline
\end{tabular}\\[.25cm]
Classification
\end{sc}
\caption{The error measures (top) and alignment values (bottom) on
four sentiment analysis domains using kernels learned as
combinations of rank-one base kernels corresponding to individual
features. The results are shown with $\pm1$ standard deviation as
measured by 5-fold cross-validation. Further measures of significance
are shown in Appendix~\ref{sec:sig},
Table~\ref{table:experiments_rank1_sig}.
}
\label{table:experiments_rank1}
\end{table}

The {\tts alignf} method returns a sparse weight vector due to the
constraint $\Mu \geq \0$. As is demonstrated by the performance of the
{\tts l1-svm} method, Table~\ref{table:experiments_rank1}, and also
previously observed by \citet{l2reg}, a sparse weight vector $\Mu$ does
not generally offer an improvement over the uniform combination in the
rank-one setting.  Thus, we focus on the performance of {\tts align} and
compare it to {\tts unif} and one-stage learning methods.
Table~\ref{table:experiments_rank1} shows that {\tts align}
significantly improves both the alignment and the error percentage over {\tt
  unif} and also improves somewhat over the one-stage {\tts l2-krr}
algorithm. Evidence of statistical significance is provided in
Appendix~\ref{sec:sig}, Table~\ref{table:experiments_rank1_sig}.  Note
that, although the sparse weighting provided by {\tts l1-svm} improves
the alignment in certain cases, it does not improve performance. 

\section{Conclusion}

We presented a series of novel algorithmic, theoretical, and empirical
results for learning kernels based on the notion of centered
alignment. Our experiments show a consistent improvement of the
performance of alignment-based algorithms over previous learning
kernel techniques, as well as the straightforward uniform kernel
combination, which has been difficult to surpass in the past, in both
classification and regression. The algorithms we described are
efficient and easy to implement. All the algorithms presented in this
paper are available in the open-source C++ library available at
{\tts www.openkernel.org}. They can be used in a variety of
applications to improve performance. We also gave an extensive
theoretical analysis which provides a number of guarantees for
centered alignment-based algorithms and methods. Several of the
algorithmic and theoretical results presented can be extended to other
learning settings. In particular, methods based on similar ideas could
be used to design learning kernel algorithms for dimensionality
reduction.

The notion of centered alignment served as a key similarity measure to
achieve these results.  Note that we are not proving that good
alignment is necessarily needed for a good classifier, but both our
theory and empirical results do suggest the existence of accurate
predictors with a good centered alignment.
\ignore{
Consider the small one
dimensional example $x_1 = -10, x_2 = -0.5, x_3 = 0.5$ and $x_4 = 10$
with labels $y_i = \sgn(x_i)$.  
Note, that even when using a linear
kernel the points are easily separated by a linear hypothesis in
feature space, however, the alignment is equal to
\begin{equation*}
\frac{\frob{\x \x^\top}{\y \y^\top}}{\|\x \x^\top\|_F \|\y \y^\top\|_F}
 = \frac{\|\x\|_1^2}{\|\x\|_2^2 4} \approx 0.55 \,.
\end{equation*}
}
Different methods based on possibly different
efficiently computable similarity measures could be used to design
effective learning kernel algorithms. In particular, the notion of
similarity suggested by \citet{balcan06}, if it could be computed from
finite samples, could be used in a equivalent way.

\acks{The work of author MM was partly supported by a Google Research Award.}

\appendix
\section{Lemmas supporting proof of Proposition~\ref{prop:1}}

For a function $f$ of the sample $S$, we denote by $\Delta(f)$ the
difference $f(S') \!- f(S)$, where $S'$ is a sample differing from $S$
by just one point, say the $m$-th point is $x_m$ in $S$ and $x'_m$ in
$S'$.  The following perturbation bound will be needed in order to
apply McDiarmid's inequality.

\begin{lemma}
\label{lemma:perturbation}
Let $\K$ and $\K'$ denote kernel matrices associated to the kernel
functions $K$ and $K'$ for a sample of size $m$ according to the
distribution $D$. Assume that for any $x \in \cX$, $K(x, x) \leq R^2$
and $K'(x, x) \leq R'^2$. Then, the following perturbation inequality
holds when changing one point of the sample:
\begin{equation*}
\frac{1}{m^2}|\Delta (\frob{\K_c}{\K'_c})| \leq \frac{24 R^2 R'^2}{m}.
\end{equation*}
\end{lemma}

\begin{proof}
By Lemma~\ref{lemma:centering}, we can write:
\begin{align*}
 \frob{\K_c}{\K'_c}  
= \frob{\K_c}{\K'}
& = \Tr\bigg[\bigg[\I - \frac{\1\1^\ttop}{m}\bigg] \K \bigg[\I - \frac{\1\1^\ttop}{m}\bigg] \K' \bigg]\\
& = \Tr\bigg[\K \K' - \frac{\1\1^\ttop}{m} \K \K' - \K \frac{\1\1^\ttop}{m} \K' + 
\frac{\1\1^\ttop}{m} \K \frac{\1\1^\ttop}{m} \K' \bigg]\\
& = \frob{\K}{\K'} - \frac{\1^\ttop (\K \K' + \K' \K) \1}{m} +
\frac{(\1^\ttop \K \1)(\1^\ttop \K' \1)}{m^2}.
\end{align*}
The perturbation of the first term is given by
\begin{equation*}
\label{eq:b1}
\Delta(\frob{\K}{\K'}) = \sum_{i = 1}^m \Delta(\K_{im} \K'_{im}) + \sum_{i \neq m}\Delta(\K_{mi} \K'_{mi}).
\end{equation*}
By the Cauchy-Schwarz inequality, for any $i, j \!\in\! [1, m]$,
\begin{equation*}
  |\K_{ij}| = |K(x_i, x_j)| \!\leq\! \sqrt{K(x_i, x_i) K(x_j, x_j)} \leq R^2
\end{equation*}
and the product can be bound as $|\K_{i,j} \K'_{i,j}| \leq |\K_{i,j}| |\K'_ij| \leq R^2
R'^2$.  The difference of products is then bound as $|\Delta
(\K_{i,j} \K'_{i,j})| \leq 2 R^2 R'^2$.
\ignore{
the difference of of products can be bound as follows,
\begin{align*}
	\Delta(\K_{ij} \K'_{ij}) & = \K_{ij}(S') \K'_{ij}(S') -
	\K_{ij}(S) \K'_{ij}(S) \\
	& \leq \K_{ij}(S') (\K'_{ij}(S) + 2R'^2) - \K_{ij}(S) \K'_{ij}(S) \\
	& = (\K_{ij}(S') - \K_{ij}(S))\K'_{ij}(S) + 2R'^2 \K_{ij}(S')  \\
	& \leq 2R^2 \K'_{ij}(S) + \K_{ij}(S') 2R'^2 \leq 4 R^2 R'^2 .
\end{align*}
The negated difference can be bounded using similar steps, which gives
$|\Delta(\K_{ij} \K'_{ij})| \leq 4 R^2 R'^2$.  }
Thus,
\begin{align*}
\frac{1}{m^2}|\Delta(\frob{\K}{\K'})|
\leq \frac{2m - 1}{m^2} (2R^2 R'^2) \leq \frac{4 R^2 R'^2}{m}.
\end{align*}
Similarly, for the first part of the second term, we obtain
\begin{align*}
\label{eq:b2}
 \frac{1}{m^2} \bigg|\Delta\bigg(\frac{\1^\ttop \K \K'\1}{m}\bigg) \bigg|
& = \bigg|\Delta\bigg(\sum_{i, j, k = 1}^m \frac{\K_{ik} \K'_{kj}}{m^3}\bigg)\bigg|\\
& = \bigg|\Delta\bigg(\frac{\sum_{i,k = 1}^m \K_{ik} \K'_{km} + \sum_{i, j \neq m}\K_{im} \K'_{mj}}{m^3} 
+ \frac{\sum_{k \neq m, j \neq m} \K_{mk} \K'_{kj}  }{m^3}\bigg)\bigg|\\
& \leq \frac{m^2 + m (m - 1) + (m - 1)^2}{m^3} (2R^2 R'^2)
\leq \frac{3m^2 - 3m + 1}{m^3} (2R^2 R'^2) \\
& \leq \frac{6 R^2 R'^2}{m}.
\end{align*}
Similarly, we have:
\begin{equation*}
\label{eq:b3}
\frac{1}{m^2} \bigg|\Delta \bigg(\frac{\1^\ttop \K' \K\1}{m}\bigg)\bigg|
\leq \frac{6 R^2 R'^2}{m}.
\end{equation*}
The final term is bounded as follows,
\begin{align*}
\label{eq:b4}
\frac{1}{m^2} \bigg|\Delta\bigg(\frac{(\1^\ttop \K \1)(\1^\ttop \K' \1)}{m^2}\bigg)\bigg| 
 & \leq  \bigg| \Delta \bigg( \frac{
   \sum_{i,j,k} \K_{ij} \K'_{km} + 
   \sum_{i,j,k \neq m} \K_{ij} \K'_{mk}
 }{m^4} ~~ + \\ 
 & \qquad \qquad \frac{
   \sum_{i, j \neq m, k \neq m} \K_{im} \K'_{jk} + 
   \sum_{i \neq m, j \neq m, k \neq m} \K_{mi} \K'_{jk}
 }{m^4} \bigg)
\bigg| \\
 & \leq \frac{m^3 + m^2 (m - 1) + m (m - 1)^2 + (m - 1)^3}{m^4} (2R^2
 R'^2) \\
 & \leq \frac{8 R^2 R'^2}{m}.
\end{align*}
Combining these last four inequalities leads directly to the statement
of the lemma.
\end{proof}

Because of the diagonal terms of the matrices,
$\frac{1}{m^2}\frob{\K_c}{\K'_c}$ is not an unbiased estimate of
$\E[K_c K'_c]$. However, as shown by the following lemma, the
estimation bias decreases at the rate $O(1/m)$.

\begin{lemma}
\label{lemma:expectation}
Under the same assumptions as Lemma~\ref{lemma:perturbation}, the
following bound on the difference of expectations holds:
\begin{equation*}
\left| \E_{x, x'}[K_c(x, x') K'_c(x, x')] -
\E_{S}\left[\frac{\frob{\K_c}{\K'_c}}{m^2}\right] \right| \leq
\frac{18 R^2 R'^2}{m}.
\end{equation*}
\end{lemma}

\begin{proof}
  To simplify the notation, unless otherwise specified, the
  expectation is taken over $x, x'$ drawn according to the
  distribution $D$.
The key observation used in this proof is that
\begin{equation}
\label{eq:39}
\E_S[\K_{ij}\K'_{ij}] = \E_S[K(x_i, x_j)K'(x_i, x_j)] = \E[KK'],
\end{equation}
for $i, j$ distinct. For expressions such as
$\E_S[\K_{ik}\K'_{kj}]$ with $i, j, k$ distinct, we obtain the
following:
\begin{equation}
\label{eq:40}
\E_S[\K_{ik}\K'_{kj}] = \E_S[K(x_i, x_k)K'(x_k, x_j)] = \E_{x'}[\E_x[K]\E_x[K']].
\end{equation}
Let us start with the expression of $\E[K_cK'_c]$:
\begin{equation}
\label{eq:41}
\E[K_cK'_c] = \E \Big[ \big(K - \E_{x'} [K] - \E_{x} [K] + \E[K]\big)
\big(K' - \E_{x'} [K'] - \E_{x} [K'] + \E[K']\big) \Big].
\end{equation}
After expanding this expression, applying the expectation to 
each of the terms, and simplifying, we obtain:
\begin{equation*}
\E[K_cK'_c] = \E[KK'] -2 \E_{x} \big[\E_{x'}[K]\E_{x'}[K']\big] + \E[K] \E[K'].
\end{equation*}
$\frob{\K_c}{\K'_c}$ can be expanded
and written more explicitly as follows:
\begin{align*}
\frob{\K_c}{\K'_c} 
& = \frob{\K}{\K'} 
- \frac{\1^\ttop \K \K' \1}{m}
- \frac{\1^\ttop \K' \K \1}{m}
+ \frac{\1^\ttop \K' \1 \1^\ttop \K \1}{m^2}\\
& = \sum_{i, j = 1}^m \K_{ij} \K'_{ij} 
- \frac{1}{m} \sum_{i, j, k = 1}^m (\K_{ik} \K'_{kj} + \K'_{ik} \K_{kj}) +
 \frac{1}{m^2} (\sum_{i, j = 1}^m \K_{ij}) (\sum_{i, j = 1}^m \K'_{ij}).
\end{align*}
To take the expectation of this expression, we use the observations
(\ref{eq:39}) and (\ref{eq:40}) and similar identities. Counting terms
of each kind, leads to the following expression of the expectation:
\begin{align*}
 \E_S\left[\frac{\frob{\K_c}{\K'_c}}{m^2}\right] 
& = \left[\frac{m (m - 1)}{m^2} - \frac{2m (m - 1)}{m^3} + \frac{2m (m - 1)}{m^4}\right] \E[KK']\\
& \qquad + \left[ \frac{-2 m (m - 1) (m -2)}{m^3} +  \frac{2 m (m - 1) (m -2)}{m^4} \right] \E_x \big[ \E_{x'}[K] \E_{x'}[K'] \big]\\
& \qquad + \left[ \frac{m (m - 1) (m -2) (m -3)}{m^4} \right] \E[K]\E[K']\\
& \qquad + \left[ \frac{m}{m^2} -  \frac{2 m}{m^3} + \frac{m}{m^4} \right] \E_x[K(x, x) K'(x, x)]\\
& \qquad + \left[ \frac{-m (m - 1)}{m^3} +  \frac{2m(m - 1)}{m^4}  \right] \E[K(x, x) K'(x, x')]\\
& \qquad + \left[ \frac{-m (m - 1)}{m^3} +  \frac{2m(m - 1)}{m^4}  \right] \E[K(x, x') K'(x, x)]\\
& \qquad + \left[ \frac{m (m - 1)}{m^4} \right] \E_x[K(x, x)]\E_x[K'(x, x)]\\
& \qquad + \left[ \frac{m (m - 1) (m - 2)}{m^4} \right] \E_x[K(x, x)]\E[K']\\
& \qquad + \left[ \frac{m (m - 1) (m - 2)}{m^4} \right] \E[K]\E_x[K'(x, x)].
\end{align*}
Taking the difference with the expression of $\E[K_cK'_c]$
(Equation \ref{eq:41}), using the fact that terms of form $\E_x[K(x,
x) K'(x, x)]$ and other similar ones are all bounded by $R^2 R'^2$ and
collecting the terms gives
\begin{align*}
 \left| \E[K_cK'_c] - \E_S\Big[\frac{\frob{\K_c}{\K'_c}}{m^2}\Big] \right|
& \leq \frac{3m^2 - 4m + 2}{m^3} \E[KK'] 
 - 2 \frac{4m^2 - 5m + 2}{m^3} \E_x \big[ \E_{x'}[K] \E_{x'}[K'] \big]\\
& \qquad + \frac{6m^2 - 11m + 6}{m^3}  \E[K]\E[K'] + \gamma,
\end{align*}
with $|\gamma| \leq  \frac{m - 1}{m^2} R^2 R'^2$. 
Using again the fact that the expectations are bounded
by $R^2 R'^2$ yields
\begin{equation*}
\left| \E[K_cK'_c] - \E_S\Big[\frac{\frob{\K_c}{\K'_c}}{m^2}\Big] \right|
\leq \left[ \frac{3}{m} +  \frac{8}{m} + \frac{6}{m} + \frac{1}{m}
\right] R^2 R'^2 \leq \frac{18}{m} R^2 R'^2,
\end{equation*}
and concludes the proof.
\end{proof}

\section{Stability bounds for alignment maximization algorithm}
\label{sec:stability}

\begin{lemma}
\label{lemma:19}
Let $\Mu \!=\! \v/\| \v \|$ and $\Mu' \!=\! \v'/\| \v' \|$. Then, the
following identity holds for $\D \Mu \!=\! \Mu' - \Mu$:
\begin{equation*}
\D \Mu
= \bigg[ \frac{\D \v}{\| \v' \|} - \frac{(\D \v)^\ttop (\v + \v') \v}
{\| \v \| \| \v' \| (\| \v \| + \| \v' \|)}\bigg].
\end{equation*}
\end{lemma}
\begin{proof}
By definition of $\D \Mu$, we can write
\begin{equation}
\label{eq:18}
\D \Mu
= \D \bigg(\frac{\v}{\| \v \|} \bigg)
= \bigg[ \frac{\v' - \v}{\| \v' \|} - \frac{\v \| \v' \| - \v \| \v \|}{\| \v \| \| \v' \|}\bigg]
=  \bigg[ \frac{\D \v}{\| \v' \|} - \frac{\v \D (\| \v \|)}{\| \v \| \| \v' \|}\bigg].
\end{equation}
Observe that: 
\begin{align*}
\D (\| \v \|) 
 = \frac{\D (\| \v \|^2)}{\| \v \| + \| \v' \|} =
\frac{\D (\sum_{i = 1}^p v_i^2)}{\| \v \| + \| \v' \|}
 = \frac{\sum_{i = 1}^p \D (v_i) (v_i + v'_i)}{\| \v \| + \| \v' \|}
 = \frac{(\D \v)^\ttop (\v + \v')}{\| \v \| + \| \v' \|}.
\end{align*}
Plugging in this expression in \eqref{eq:18} yields the statement of the lemma.
\end{proof}
Consider the minimization~\eqref{eq:10} shown by Proposition~\ref{prop:3}
to provide the solution of the alignment maximization problem for 
a convex combination. The matrix $\M$ and vector $\a$ are functions of 
the training sample $S$. To emphasize this dependency, we rewrite
that optimization for a sample $S$ as
\begin{equation}
\label{eq:19}
\min_{\v \geq \0}  F(S, \v),
\end{equation}
where $F(S, \v) = \v^\ttop \M \v - 2 \v^\ttop \a = \| \v\|_\M^2 - 2 \v^\ttop \a$.
The following lemma provides a stability result for this optimization
problem.

\begin{proposition}
\label{prop:20}
  Let $S$ and $S'$ denote two samples of size $m$ differing by only
  one point.  Let $\v$ and $\v'$ be the solution of \eqref{eq:19},
  respectively, for sample $S$ and $S'$. Then, the following
  inequality holds for $\D \v = \v' - \v$:
\begin{equation*}
\| \D \v \|_\M^2 \leq  \big[ \D \a -  (\D \M)\v' \big]^{\ttop} \D \v.
\end{equation*}
\end{proposition}

\begin{proof}
  Since $C = \set{\v\colon \v \geq 0}$ is convex, for any $s \in [0,
  1]$, $\v + s \D \v$ and $\v' - s \D \v$ are in $C$. Thus, by
  definition of $\v'$ and $\v$,
\begin{equation*}
F(S, \v) \leq F(S, \v + s \D \v) \quad \text{and} \quad F(S', \v') \leq F(S', \v' - s \D \v).
\end{equation*}
Summing up these inequalities, we obtain
\begin{multline*}
\| \v \|^2_\M - \| \v +s \D\v \|^2_\M + 
\| \v' \|^2_{\M'} - \| \v' - s\D\v \|^2_{\M'} \\
\begin{aligned}
& \leq
2 \v^\ttop \a - 2(\v + s\D \v)^\ttop \a +
2 \v'^\ttop \a' - 2(\v' + s\D \v)^\ttop \a' \\
& = -2 [s \a^\ttop \D \v - s \a'^\ttop \D \v] = 2s (\D \a)^\ttop \D \v.
\end{aligned}
\end{multline*}
The left-hand side of this inequality can be rewritten as follows
after expansion and using the identity $\| \v' - s \D \v\|_{\M'}^2 -
\| \v' - s \D \v\|_{\M}^2 = \| \v' - s \D \v\|_{\D \M}^2$:
\begin{multline*}
  - \|s \D \v\|_\M^2 - 2s \v^\ttop \M \D \v +  \| \v' \|_{\M'}^2 - \| \v'
\|_\M^2 - \| s \D \v \|_\M^2 + 2 s \v'^\ttop \M (\D \v) - \| \v' - s \D
\v\|_{\D \M}^2 \\
   = 2s(1-s) \|\D \v\|_\M^2 + \|\v'\|_{\D \M}^2 - \|
\v' - s \D \v\|_{\D \M}^2 .
\end{multline*}
Then, expanding $\| \v' - s \D \v\|_{\D \M}^2$ results in the final
inequality
\begin{equation*}
2 s(1-s) \| \D \v \|_\M^2 - s^2 \| \D \v \|_{\D \M}^2 + 2 s \v'^\ttop (\D
\M) (\D \v) 
\leq  2s (\D \a)^\ttop \D \v.
\end{equation*}
Dividing by $s$ and setting $s \!=\! 0$ yields
\begin{equation*}
\| \D \v \|_\M^2 + \v'^\ttop (\D \M) (\D \v) \leq (\D \a)^\ttop \D \v,
\end{equation*}
which concludes the proof of the lemma.
\end{proof}

\section{Significance tests for empirical results}
\label{sec:sig}

\begin{table}
\begin{sc}
\begin{center}
\begin{tabular}{cc}
  Kinematics & Ionosphere \\
  \begin{tabular}{c|cccc}
   & \begin{sideways}{\tts unif}\end{sideways} & \begin{sideways}{\tts l2-krr}\end{sideways} & \begin{sideways}{\tts align}\end{sideways} & \begin{sideways}{\tts alignf}\end{sideways}\\
  \hline
  {\tts unif} & -- & 1 & 1 & 1 \\
  {\tts l2-krr} & 0 & -- & 1 & 1 \\
  {\tts align} & 0 & 0 & -- & 1 \\
  {\tts alignf} & 0 & 0 & 0 & -- \\
  \end{tabular}
&
  \begin{tabular}{c|cccc}
   & \begin{sideways}{\tts unif}\end{sideways} & \begin{sideways}{\tts l2-krr}\end{sideways} & \begin{sideways}{\tts align}\end{sideways} & \begin{sideways}{\tts alignf}\end{sideways}\\
  \hline
  {\tts unif} & -- & 1 & 1 & 1 \\
  {\tts l2-krr} & 0 & -- & 1 & 1 \\
  {\tts align} & 0 & 0 & -- & 1 \\
  {\tts alignf} & 0 & 0 & 0 & -- \\
  \end{tabular}
\end{tabular} \vspace{0.5cm} \\
\begin{tabular}{ccc}
German & Spambase & Splice \\
  \begin{tabular}{c|cccc}
   & \begin{sideways}{\tts unif}\end{sideways} & \begin{sideways}{\tts l1-svm}\end{sideways} & \begin{sideways}{\tts align}\end{sideways} & \begin{sideways}{\tts alignf}\end{sideways}\\
  \hline
  {\tts unif} & -- & 0 & 1 & 1 \\
  {\tts l1-svm} & 0 & -- & 0 & 1 \\
  {\tts align} & 0 & 0 & -- & 1 \\
  {\tts alignf} & 0 & 0 & 0 & -- \\
  \end{tabular}
&
  \begin{tabular}{c|cccc}
   & \begin{sideways}{\tts unif}\end{sideways} & \begin{sideways}{\tts l1-svm}\end{sideways} & \begin{sideways}{\tts align}\end{sideways} & \begin{sideways}{\tts alignf}\end{sideways}\\
  \hline
  {\tts unif} & -- & 0 & 0 & 0 \\
  {\tts l1-svm} & 1 & -- & 1 & 1 \\
  {\tts align} & 0 & 0 & -- & 0 \\
  {\tts alignf} & 0 & 0 & 0 & -- \\
  \end{tabular}
& 
  \begin{tabular}{c|cccc}
   & \begin{sideways}{\tts unif}\end{sideways} & \begin{sideways}{\tts l1-svm}\end{sideways} & \begin{sideways}{\tts align}\end{sideways} & \begin{sideways}{\tts alignf}\end{sideways}\\
  \hline
  {\tts unif} & -- & 0 & 0 & 1 \\
  {\tts l1-svm} & 0 & -- & 0 & 0 \\
  {\tts align} & 0 & 0 & -- & 0 \\
  {\tts alignf} & 0 & 0 & 0 & -- \\
  \end{tabular}
\end{tabular}
\end{center}
\end{sc}
\caption{Significance tests for general kernel combination results
presented in Table~\ref{table:experiments}. An entry of 1
indicates that the algorithm listed in the column has a significantly
better accuracy than the algorithm listed in the row.}
\label{table:experiments_sig}
\end{table}

\begin{table}
\begin{sc}
\begin{center}
\begin{tabular}{cccc}
  Books & Dvd & Elec & Kitchen \\
  \begin{tabular}{c|c@{~~~}c@{~~~}c}
   & \begin{sideways}{\tts unif}\end{sideways} & \begin{sideways}{\tts l2-krr}\end{sideways} & \begin{sideways}{\tts align}\end{sideways}\\
  \hline
  {\tts unif} & -- & 1 & 1 \\
  {\tts l2-krr} & 0 & -- & 1 \\
  {\tts align} & 0 & 0 & -- \\
  \end{tabular}
&
  \begin{tabular}{c|c@{~~~}c@{~~~}c}
   & \begin{sideways}{\tts unif}\end{sideways} & \begin{sideways}{\tts l2-krr}\end{sideways} & \begin{sideways}{\tts align}\end{sideways}\\
  \hline
  {\tts unif} & -- & 1 & 1 \\
  {\tts l2-krr} & 0 & -- & 0 \\
  {\tts align} & 0 & 0 & -- \\
  \end{tabular}
&
  \begin{tabular}{c|c@{~~~}c@{~~~}c}
   & \begin{sideways}{\tts unif}\end{sideways} & \begin{sideways}{\tts l2-krr}\end{sideways} & \begin{sideways}{\tts align}\end{sideways}\\
  \hline
  {\tts unif} & -- & 1 & 1 \\
  {\tts l2-krr} & 0 & -- & 1 \\
  {\tts align} & 0 & 0 & -- \\
  \end{tabular}
&
  \begin{tabular}{c|c@{~~~}c@{~~~}c}
   & \begin{sideways}{\tts unif}\end{sideways} & \begin{sideways}{\tts l2-krr}\end{sideways} & \begin{sideways}{\tts align}\end{sideways}\\
  \hline
  {\tts unif} & -- & 1 & 1 \\
  {\tts l2-krr} & 0 & -- & 1 \\
  {\tts align} & 0 & 0 & -- \\
  \end{tabular}
\end{tabular} \\
Regression \vspace{.5cm} \\

  \begin{tabular}{c|c@{~~~}c@{~~~}c}
  Books & Dvd & Elec & Kitchen \\
  \begin{tabular}{c|ccc}
   & \begin{sideways}{\tts unif}\end{sideways} & \begin{sideways}{\tts l1-svm}\end{sideways} & \begin{sideways}{\tts align}\end{sideways}\\
  \hline
  {\tts unif} & -- & 0 & 1 \\
  {\tts l1-svm} & 1 & -- & 1 \\
  {\tts align} & 0 & 0 & -- \\
  \end{tabular}
&
  \begin{tabular}{c|c@{~~~}c@{~~~}c}
   & \begin{sideways}{\tts unif}\end{sideways} & \begin{sideways}{\tts l1-svm}\end{sideways} & \begin{sideways}{\tts align}\end{sideways}\\
  \hline
  {\tts unif} & -- & 0 & 1 \\
  {\tts l1-svm} & 1 & -- & 1 \\
  {\tts align} & 0 & 0 & -- \\
  \end{tabular}
&
  \begin{tabular}{c|c@{~~~}c@{~~~}c}
   & \begin{sideways}{\tts unif}\end{sideways} & \begin{sideways}{\tts l1-svm}\end{sideways} & \begin{sideways}{\tts align}\end{sideways}\\
  \hline
  {\tts unif} & -- & 0 & 1 \\
  {\tts l1-svm} & 1 & -- & 1 \\
  {\tts align} & 0 & 0 & -- \\
  \end{tabular}
&
  \begin{tabular}{c|c@{~~~}c@{~~~}c}
   & \begin{sideways}{\tts unif}\end{sideways} & \begin{sideways}{\tts l1-svm}\end{sideways} & \begin{sideways}{\tts align}\end{sideways}\\
  \hline
  {\tts unif} & -- & 0 & 1 \\
  {\tts l1-svm} & 1 & -- & 1 \\
  {\tts align} & 0 & 0 & -- \\
  \end{tabular}
\end{tabular} \\
Classification
\end{center}
\end{sc}
\caption{Significance tests for rank-one kernel combination results
presented in Table~\ref{table:experiments_rank1}.  An entry of 1
indicates that the algorithm listed in the column has a significantly
better accuracy then the algorithm listed in the row.}
\label{table:experiments_rank1_sig}
\end{table}

Tables~\ref{table:experiments_sig} and
\ref{table:experiments_rank1_sig} show the results of paired-sample
one-sided T-tests for all pairs
of algorithms compared across all datasets presented in
Section~\ref{sec:experiments} for both regression and classification.
Each entry of the tables indicates whether the mean error of the
algorithm listed in the column is significantly less than the mean
error of the algorithm listed in the row at significance level
$p=0.1$. An entry of 1 indicates a significant difference, while an
entry of 0 indicates that the null hypothesis (that the errors are not
significantly different) cannot be rejected.

Table~\ref{table:experiments_sig} indicates that the {\tts alignf}
method offers significant improvement over ${\tts unif}$ in all
datasets with the exception of spambase and significantly improves
over the compared one-stage method in all datasets apart from splice.
Table~\ref{table:experiments_rank1_sig} indicates that the {\tts
align} method significantly improves over both the uniform and one-stage
combination in all datasets apart from dvd in the regression setting,
where improvement over {\tts l2-krr} is not deemed significant.

\newpage
\bibliography{alignj}

\begin{thebibliography}{41}
\providecommand{\natexlab}[1]{#1}
\providecommand{\url}[1]{\texttt{#1}}
\expandafter\ifx\csname urlstyle\endcsname\relax
  \providecommand{\doi}[1]{doi: #1}\else
  \providecommand{\doi}{doi: \begingroup \urlstyle{rm}\Url}\fi

\bibitem[Argyriou et~al.(2005)Argyriou, Micchelli, and Pontil]{argyriou_colt}
Andreas Argyriou, Charles Micchelli, and Massimiliano Pontil.
\newblock Learning convex combinations of continuously parameterized basic
  kernels.
\newblock In \emph{COLT}, 2005.

\bibitem[Argyriou et~al.(2006)Argyriou, Hauser, Micchelli, and
  Pontil]{argyriou_icml}
Andreas Argyriou, Raphael Hauser, Charles Micchelli, and Massimiliano Pontil.
\newblock A {DC}-programming algorithm for kernel selection.
\newblock In \emph{ICML}, 2006.

\bibitem[Bach(2008)]{bach}
Francis Bach.
\newblock {Exploring large feature spaces with hierarchical multiple kernel
  learning}.
\newblock In \emph{NIPS}, 2008.

\bibitem[Balcan and Blum(2006)]{balcan06}
Maria-Florina Balcan and Avrim Blum.
\newblock On a theory of learning with similarity functions.
\newblock In \emph{ICML}, 2006.

\bibitem[Blitzer et~al.(2007)Blitzer, Dredze, and
  Pereira]{Blitzer07Biographies}
John Blitzer, Mark Dredze, and Fernando Pereira.
\newblock {Biographies, Bollywood, Boom-boxes and Blenders: Domain Adaptation
  for Sentiment Classification}.
\newblock In \emph{ACL}, 2007.

\bibitem[Boser et~al.(1992)Boser, Guyon, and Vapnik]{bgv}
Bernhard Boser, Isabelle Guyon, and Vladimir Vapnik.
\newblock A training algorithm for optimal margin classifiers.
\newblock In \emph{COLT}, volume~5, 1992.

\bibitem[Bousquet and Elisseeff(2000)]{bousquet}
Olivier Bousquet and Andr\'e Elisseeff.
\newblock Algorithmic stability and generalization performance.
\newblock In \emph{NIPS}, 2000.

\bibitem[Bousquet and Herrmann(2002)]{bousquet_and_herrmann}
Olivier Bousquet and Daniel J.~L. Herrmann.
\newblock On the complexity of learning the kernel matrix.
\newblock In \emph{NIPS}, 2002.

\bibitem[Boyd and Vandenberghe(2004)]{boyd}
Stephen Boyd and Lieven Vandenberghe.
\newblock \emph{Convex Optimization}.
\newblock Cambridge University Press, 2004.

\bibitem[Chapelle et~al.(2002)Chapelle, Vapnik, Bousquet, and
  Mukherjee]{chapelle_et_al}
Olivier Chapelle, Vladimir Vapnik, Olivier Bousquet, and Sayan Mukherjee.
\newblock Choosing multiple parameters for support vector machines.
\newblock \emph{Machine Learning}, 46\penalty0 (1-3), 2002.

\bibitem[Cortes(2009)]{cortes09}
Corinna Cortes.
\newblock Invited talk: Can learning kernels help performance?
\newblock In \emph{ICML}, 2009.

\bibitem[Cortes and Vapnik(1995)]{ccvv}
Corinna Cortes and Vladimir Vapnik.
\newblock {Support-Vector Networks}.
\newblock \emph{Machine Learning}, 20\penalty0 (3), 1995.

\bibitem[Cortes et~al.(2008)Cortes, Mohri, and Rostamizadeh]{lsk}
Corinna Cortes, Mehryar Mohri, and Afshin Rostamizadeh.
\newblock Learning sequence kernels.
\newblock In \emph{MLSP}, 2008.

\bibitem[Cortes et~al.(2009{\natexlab{a}})Cortes, Mohri, and
  Rostamizadeh]{l2reg}
Corinna Cortes, Mehryar Mohri, and Afshin Rostamizadeh.
\newblock {$L_2$}-regularization for learning kernels.
\newblock In \emph{UAI}, 2009{\natexlab{a}}.

\bibitem[Cortes et~al.(2009{\natexlab{b}})Cortes, Mohri, and Rostamizadeh]{nlk}
Corinna Cortes, Mehryar Mohri, and Afshin Rostamizadeh.
\newblock Learning non-linear combinations of kernels.
\newblock In \emph{NIPS}, 2009{\natexlab{b}}.

\bibitem[Cortes et~al.(2010{\natexlab{a}})Cortes, Mohri, and
  Rostamizadeh]{align}
Corinna Cortes, Mehryar Mohri, and Afshin Rostamizadeh.
\newblock Two-stage learning kernel methods.
\newblock In \emph{ICML}, 2010{\natexlab{a}}.

\bibitem[Cortes et~al.(2010{\natexlab{b}})Cortes, Mohri, and Rostamizadeh]{lk}
Corinna Cortes, Mehryar Mohri, and Afshin Rostamizadeh.
\newblock Generalization bounds for learning kernels.
\newblock In \emph{ICML}, 2010{\natexlab{b}}.

\bibitem[Cortes et~al.(2010{\natexlab{c}})Cortes, Mohri, and Talwalkar]{approx}
Corinna Cortes, Mehryar Mohri, and Ameet Talwalkar.
\newblock {On the Impact of Kernel Approximation on Learning Accuracy}.
\newblock In \emph{AISTATS}, 2010{\natexlab{c}}.

\bibitem[Cortes et~al.(2011{\natexlab{a}})Cortes, Mohri, and Rostamizadeh]{ek}
Corinna Cortes, Mehryar Mohri, and Afshin Rostamizadeh.
\newblock Ensembles of kernel predictors.
\newblock In \emph{UAI}, 2011{\natexlab{a}}.

\bibitem[Cortes et~al.(2011{\natexlab{b}})Cortes, Mohri, and
  Rostamizadeh]{tutorial11}
Corinna Cortes, Mehryar Mohri, and Afshin Rostamizadeh.
\newblock Tutorial: Learning kernels.
\newblock In \emph{ICML}, 2011{\natexlab{b}}.

\bibitem[Cristianini et~al.(2001)Cristianini, Shawe-Taylor, Elisseeff, and
  Kandola]{align-nips}
Nello Cristianini, John Shawe-Taylor, Andr\'e Elisseeff, and Jaz~S. Kandola.
\newblock On kernel-target alignment.
\newblock In \emph{NIPS}, 2001.

\bibitem[Cristianini et~al.(2002)Cristianini, Kandola, Elisseeff, and
  Shawe-Taylor]{align-unpublished}
Nello Cristianini, Jaz~S. Kandola, Andr\'e Elisseeff, and John Shawe-Taylor.
\newblock On kernel target alignment.
\newblock http://www.support-vector.net/papers/alignment\_JMLR.ps, unpublished,
  2002.

\bibitem[Gretton et~al.(2005)Gretton, Bousquet, Smola, and
  Sch{\"o}lkopf]{gretton2005}
Arthur Gretton, Olivier Bousquet, Alexander Smola, and Bernhard Sch{\"o}lkopf.
\newblock Measuring statistical dependence with {H}ilbert-{S}chmidt norms.
\newblock In \emph{Algorithmic learning theory}, 2005.

\bibitem[Jebara(2004)]{jebara04}
Tony Jebara.
\newblock {Multi-task feature and kernel selection for SVMs}.
\newblock In \emph{ICML}, 2004.

\bibitem[Kandola et~al.(2002{\natexlab{a}})Kandola, Shawe-Taylor, and
  Cristianini]{align-tech-report}
Jaz~S. Kandola, John Shawe-Taylor, and Nello Cristianini.
\newblock On the extensions of kernel alignment.
\newblock technical report 120, Department of Computer Science, Univ. of
  London, UK, 2002{\natexlab{a}}.

\bibitem[Kandola et~al.(2002{\natexlab{b}})Kandola, Shawe-Taylor, and
  Cristianini]{align-tech-report-2}
Jaz~S. Kandola, John Shawe-Taylor, and Nello Cristianini.
\newblock Optimizing kernel alignment over combinations of kernels.
\newblock technical report 121, Dept. of CS, Univ. of London, UK,
  2002{\natexlab{b}}.

\bibitem[Kim et~al.(2006)Kim, Magnani, and Boyd]{kim2006}
Seung-Jean Kim, Alessandro Magnani, and Stephen Boyd.
\newblock Optimal kernel selection in kernel fisher discriminant analysis.
\newblock In \emph{ICML}, 2006.

\bibitem[Koltchinskii and Yuan(2008)]{koltchinskii2008}
Vladimir Koltchinskii and Ming Yuan.
\newblock Sparse recovery in large ensembles of kernel machines.
\newblock In \emph{COLT}, 2008.

\bibitem[Lanckriet et~al.(2004)Lanckriet, Cristianini, Bartlett, Ghaoui, and
  Jordan]{lanckriet}
Gert Lanckriet, Nello Cristianini, Peter Bartlett, Laurent~El Ghaoui, and
  Michael Jordan.
\newblock Learning the kernel matrix with semidefinite programming.
\newblock \emph{JMLR}, 5, 2004.

\bibitem[Lewis et~al.(2006)Lewis, Jebara, and Noble]{lewis_et_al}
Darrin~P. Lewis, Tony Jebara, and William~Stafford Noble.
\newblock Nonstationary kernel combination.
\newblock In \emph{ICML}, 2006.

\bibitem[McDiarmid(1989)]{mcdiarmid}
Colin McDiarmid.
\newblock {On the method of bounded differences}.
\newblock \emph{Surveys in combinatorics}, 141, 1989.

\bibitem[Meila(2003)]{meila}
Marina Meila.
\newblock Data centering in feature space.
\newblock In \emph{AISTATS}, 2003.

\bibitem[Micchelli and Pontil(2005)]{micchelli_and_pontil}
Charles Micchelli and Massimiliano Pontil.
\newblock Learning the kernel function via regularization.
\newblock \emph{JMLR}, 6, 2005.

\bibitem[Ong et~al.(2005)Ong, Smola, and Williamson]{ong}
Cheng~Soon Ong, Alexander Smola, and Robert Williamson.
\newblock Learning the kernel with hyperkernels.
\newblock \emph{JMLR}, 6, 2005.

\bibitem[Pothin and Richard(2008)]{pothin}
Jean-Baptiste Pothin and C\'edric Richard.
\newblock Optimizing kernel alignment by data translation in feature space.
\newblock In \emph{ICASSP}, 2008.

\bibitem[Saunders et~al.(1998)Saunders, Gammerman, and Vovk]{krr}
Craig Saunders, A.~Gammerman, and Volodya Vovk.
\newblock Ridge regression learning algorithm in dual variables.
\newblock In \emph{ICML}, 1998.

\bibitem[Sonnenburg et~al.(2006)Sonnenburg, R\"{a}tsch, Sch\"{a}fer, and
  Sch\"{o}lkopf]{sonnenburg}
S\"{o}ren Sonnenburg, Gunnar R\"{a}tsch, Christin Sch\"{a}fer, and Bernhard
  Sch\"{o}lkopf.
\newblock Large scale multiple kernel learning.
\newblock \emph{Journal of Machine Learning Research}, 7:\penalty0 1531--1565,
  2006.

\bibitem[Srebro and Ben-David(2006)]{shai}
Nathan Srebro and Shai Ben-David.
\newblock Learning bounds for support vector machines with learned kernels.
\newblock In \emph{COLT}, 2006.

\bibitem[Vapnik(1998)]{vapnik98}
Vladimir~N. Vapnik.
\newblock \emph{Statistical Learning Theory}.
\newblock John Wiley \& Sons, 1998.

\bibitem[Varma and Babu(2009)]{varma}
Manik Varma and Bodla~Rakesh Babu.
\newblock More generality in efficient multiple kernel learning.
\newblock In \emph{ICML}, 2009.

\bibitem[Zien and Ong(2007)]{zienO07}
Alexander Zien and Cheng~Soon Ong.
\newblock Multiclass multiple kernel learning.
\newblock In \emph{ICML}, 2007.

\end{thebibliography}
\end{document}